\documentclass[sigconf]{acmart}

\AtBeginDocument{%
  \providecommand\BibTeX{{%
    \normalfont B\kern-0.5em{\scshape i\kern-0.25em b}\kern-0.8em\TeX}}}

\copyrightyear{2020} 
\acmYear{2020} 
\setcopyright{acmlicensed}
\acmConference[WSDM '20]{The Thirteenth ACM International Conference on Web Search and Data Mining}{February 3--7, 2020}{Houston, TX, USA}
\acmBooktitle{The Thirteenth ACM International Conference on Web Search and Data Mining (WSDM '20), February 3--7, 2020, Houston, TX, USA}
\acmPrice{15.00}
\acmDOI{10.1145/3336191.3371768}
\acmISBN{978-1-4503-6822-3/20/02}



\usepackage{booktabs}
\usepackage{graphicx}
\usepackage{xspace}
\usepackage{textcomp}
\usepackage{algpseudocode}
\usepackage{algorithm}

\newtheorem{theorem}{Theorem}

\def\vec#1{\mathchoice%
	{\mbox{\boldmath $\displaystyle\bf#1$}}
	{\mbox{\boldmath $\textstyle\bf#1$}}
	{\mbox{\boldmath $\scriptstyle\bf#1$}}
	{\mbox{\boldmath $\scriptscriptstyle\bf#1$}}}
\def\v#1{\protect\vec #1}

\newcommand{\TODOK}[1]{
\ifmmode
\text{\textcolor{blue}{ }}
\else
\textcolor{blue}{ }
\fi
}
\newcommand{\TODOA}[1]{
\ifmmode
\text{\textcolor{red}{ }}
\else
\textcolor{red}{ }
\fi
}

\renewcommand{\vec}[1]{{\mathbf{#1}}}

\numberwithin{theorem}{section}
\newtheorem{lemma}[theorem]{Lemma}

\newcommand{\alg}{XReg\xspace}

\newcommand{\myhref}[3][blue]{\href{#2}{\color{#1}{#3}}}%

\makeatletter
\newcommand{\newreptheorem}[2]{\newtheorem*{rep@#1}{\rep@title}
\newenvironment{rep#1}[1]{\def\rep@title{#2 \ref*{##1}}\begin{rep@#1}}{\end{rep@#1}}
}
\makeatother

\newreptheorem{lemma}{Lemma}
\newreptheorem{theorem}{Theorem}
\newreptheorem{claim}{Claim}

\makeatletter
\newcount\my@repeat@count
\newcommand{\myrep}[2]{%
	\begingroup
	\my@repeat@count=\z@
	\@whilenum\my@repeat@count<#1\do{#2\advance\my@repeat@count\@ne}%
	\endgroup
}
\makeatother

\newcommand{\myneg}[1]{\myrep{#1}{\!}}
\newcommand{\mypos}[1]{\myrep{#1}{\,}}
\newcommand{\suppurl}{\myhref{http://manikvarma.org/pubs/prabhu20-supp.pdf}{supplementary\xspace}}

\begin{document}


\title[Extreme Regression for Dynamic Search Advertising]{Extreme Regression for Dynamic Search Advertising}
\author[Y. Prabhu]{Yashoteja Prabhu}
\email{t-yaprab@microsoft.com}\authornote{Microsoft Research India}\authornote{Indian Institute of Technology Delhi}
\affiliation{}
\author[A. Kusupati]{Aditya Kusupati}
\email{kusupati@cs.washington.edu}\authornote{University of Washington}\authornote{Work done during Research Fellowship at Microsoft Research India}
\affiliation{}
\author[N. Gupta]{Nilesh Gupta}\authornotemark[1]
\email{t-nilgup@microsoft.com}
\affiliation{}
\author[M. Varma]{Manik Varma}
\email{manik@microsoft.com}\authornotemark[1]\authornotemark[2]
\affiliation{}

\begin{abstract}
This paper introduces a new learning paradigm called eXtreme Regression (XR) whose objective is to accurately predict the numerical degrees of relevance of an extremely large number of labels to a data point. XR can provide elegant solutions to many large-scale ranking and recommendation applications including Dynamic Search Advertising (DSA). XR can learn more accurate models than the recently popular extreme classifiers which incorrectly assume strictly binary-valued label relevances. Traditional regression metrics which sum the errors over all the labels are unsuitable for XR problems since they could give extremely loose bounds for the label ranking quality. Also, the existing regression algorithms won't efficiently scale to millions of labels. This paper addresses these limitations through: (1) {\it new evaluation metrics} for XR which sum only the $k$ largest regression errors; (2) a {\it new algorithm} called \alg\! which decomposes XR task into a hierarchy of much smaller regression problems thus leading to highly efficient training and prediction. This paper also introduces a (3) {\it new labelwise prediction algorithm} in \alg useful for DSA and other recommendation tasks.

Experiments on benchmark datasets demonstrated that \alg can outperform the state-of-the-art extreme classifiers as well as large-scale regressors and rankers by up to 50\% reduction in the new XR error metric, and up to 2\% and 2.4\% improvements in terms of the propensity-scored precision metric used in extreme classification and the click-through rate metric used in DSA respectively. Deployment of \alg on DSA in Bing resulted in a relative gain of 27\% in query coverage. \alg's source code can be downloaded from \citep{XRegCode}
\end{abstract}
\maketitle
\section{Introduction}
\label{sec:introduction}
{\bf Objective}: This paper introduces a new learning paradigm called eXtreme Regression (XR) which can provide elegant solutions to many large-scale ranking and recommendation applications including Dynamic Search Advertising (DSA). To effectively solve XR problems, this paper also develops new evaluation metrics and a new highly scalable and accurate algorithm called \alg.

{\bf eXtreme Regression}: The objective of eXtreme Regression is to learn to accurately predict the numerical degrees of relevance of an extremely large number of labels with respect to a data point. Many large-scale ranking and recommendation applications can naturally be reformulated as XR problems. For example, the tasks of DSA, movie recommendation and document tagging can be posed as the problems of predicting the search queries' click probabilities for an ad, the users' ratings for a movie and the informativeness of tags while describing a web document, respectively. These qualify as XR problems since the total number of queries, users and tags can potentially be in millions in these applications. The predicted relevance estimates could then be used to recommend the most relevant labels to a data point which is the desired end goal of recommendation systems. Alternatively, the recommendations can also be further refined by filtering off less relevant ones or by re-ranking them to improve their relevance, and the relevance estimates provide principled ways of achieving these. To successfully solve an XR problem, new algorithms which could train and predict efficiently over millions of labels as well as millions of data points while also maintaining high prediction accuracy are required. Furthermore, the definition of accuracy, or equivalently regression error, needs to be redefined for XR settings where both the relevant labels and the desired label recommendations are extremely small in number compared to the complete label set whose most labels have no influence on final recommendations. This paper addresses these challenges by developing new evaluation metrics and algorithms.

{\bf DSA}: DSA is a format of search advertising where the ads to be shown against a search query, along with the associated ad-copy, ad-title, bid-phrases {\it etc.}, are algorithmically obtained by leveraging the content from the ad landing pages. This saves considerable efforts for advertisers, results in faster deployment of new ad campaigns and enables more accurate user targeting. The ads shown by DSA algorithms need to be highly relevant and generate user clicks for the given query in order to earn revenue for the search engine and satisfy the users and advertisers. In addition, these algorithms need to train and predict very efficiently in order to scale to billions of ads and millions of search queries across multiple markets and maintain milliseconds' prediction latencies. This paper solves DSA as an XR task of estimating the click probabilities for the query, ad pairs by using the new \alg algorithm. Note that different ads can have different click probabilities for same query owing to multiple query intents. For example the query "throne" on Bing refers to an \href{https://plarium.com/landings/en/throne/top_s}{\color{blue}{online strategy war game}}, an \href{https://hbo.com/game-of-thrones}{\color{blue}{online tv series}} and a \href{https://ebay.com/sch/i.html?_nkw=throne+chair}{\color{blue}{furniture product}} with click probabilities of 0.2, 0.06 and 0.004 respectively. Based on the predicted click probabilities, the less clickable ads are filtered off, the remaining ads are re-ranked to promote those of high quality and high advertiser bids, and a small number of top ranked ads are finally shown for the given query. 

{\bf Extreme Classification}: Extreme classifiers annotate a data point with the most relevant {\it subset} of labels from an extremely large label set. Owing to their high scalability and accuracy in label subset selection scenarios, extreme classifiers are increasingly being used for DSA~\citep{Prabhu18b} and other large-scale recommendation problems. Unfortunately, they make a fundamentally incorrect assumption that a label is either fully relevant or fully irrelevant to a data point which hurts their model accuracy. When applied to DSA, they approximate all click-through rates to either 0 or 1 during training and thus end up predicting less clickable ads. In turn, this also undermines further filtering and re-ranking steps due to the lack of reliable click probability estimates. Also, the ranking at the top metrics used for evaluating extreme classifiers ignore the errors in estimating the relevances and are hence not suitable for XR.

{\bf Regression and ranking}: Multivariate regressors predict multiple numerical outcome variables as functions of the features of a data point. Although such regressors could reliably estimate the label relevances in XR, most existing regressors are designed for small number of outcome variables and do not scale to millions of labels in XR. Moreover, the standard regression metrics such as Mean Absolute Deviation (MAD) which sum the regression errors over all the labels are unsuitable for XR problems because the quality of recommended labels, both before and after the filtering and re-ranking steps, depend only on the accurate estimation of a small number of label relevances. The pairwise ranking approaches, which ensure that a more relevant item is ranked ahead of a less relevant one for each pair of items, have been extensively used for moderate-sized ranking and recommendation tasks. However, their complexity scales quadratically in number of labels and therefore don't scale to million labels. 

{\bf eXtreme Regression metrics}: This paper proposes new regression metrics for XR which serve as good proxies for the ranking accuracy and for the qualities of the subsequent label filtering and re-ranking steps. These metrics average of the largest few regression errors which are usually caused by highly underestimating or highly overestimating the relevances of the most or the least relevant labels which in turn degrade the ranking quality. The new XMAD@$k$ metric can give up to 69x tighter bounds over ranking regret than MAD. These new metrics can guide the crucial steps in XR such as training, performance evaluation, hyper-parameter tuning, model selection {\it etc.}

{\bf eXtreme Regressor algorithm}: This paper also develops a new eXtreme Regressor (\alg) algorithm which can efficiently regress on to millions of label relevance weights in only logarithmic time. \alg hierarchically clusters the labels into a balanced tree and learns approximate regressors in each tree node which are common to all the labels in the node. Due to high label sparsity, each data point only participates in a logarithmic number of tree nodes which can lead to a significant speed up during both training and prediction by using appropriate algorithms. \alg essentially extends the state-of-the-art Parabel extreme classifier to the regression setting. \alg consistently outperforms extreme classifiers, large-scale regressors and rankers in terms of ranking accuracy. On a DSA dataset with 5M ads \& 1M queries, \alg can train within just 20 hours using 1 core, predict in just 3 ms per query and give up to 27\% lift in query coverage when deployed online. 

{\bf Labelwise inference}: The standard prediction scenario involves recommending the most relevant labels for a test point, referred here as pointwise prediction, but applications such as DSA and movie recommendation can more naturally be posed in the reverse manner of predicting the most relevant ads or movies ({\it i.e.} test points) for each query or user ({\it i.e.} each label), referred here as labelwise prediction. On these tasks, pointwise prediction might recommend a small set of highly popular labels that are relevant to all test points resulting in low label coverage. This paper develops an efficient labelwise prediction algorithm in \alg, which significantly improves the query coverage in DSA. Note that the \alg training is agnostic to the choice of the prediction setting and the same learnt model works well for both types of predictions.

{\bf Contributions:} This paper: (a) introduces a new learning paradigm called eXtreme Regression (XR) and reformulates tagging, movie recommendation and DSA applications as XR problems; (b) develops new evaluation metrics and a highly scalable and accurate algorithm called \alg to effectively tackle XR problems; and (c) demonstrates that \alg can significantly improve query coverage on Bing DSA when deployed in production. \alg's source code can be downloaded from \citep{XRegCode}.

\section{Related Work}
\label{related_work}

{\bf Extreme Classification}: Much progress has recently been made in developing extreme multi-label classifiers based on trees~\citep{Agrawal13,Prabhu14,Jain16,Jasinska16,Si17,Prabhu18}, embeddings~\citep{Hsu09,Weston11,Chen12,Cisse13,Lin14,Bhatia15,Mineiro15,Tagami17,Xu16,Guo19} and 1-vs-all approaches~\citep{Weston13,Yen16,Babbar17,Niculescu17,Liu17,Yen17,Babbar18,Siblini18,Prabhu18b,Jain19,You19}. Among these, 1-vs-all approaches like  DiSMEC~\citep{Babbar17}, ProXML~\citep{Babbar18}, Parabel~\citep{Prabhu18b} and Slice~\citep{Jain19} achieve state-of-the-art results on Precision@$k$, nDCG@$k$ and their propensity-scored counterparts, but train only from binary labels and are hence not apt for DSA. In terms of efficiency, Parabel is many orders faster to train and predict than DiSMEC and ProXML, hence \alg algorithm builds on top of Parabel. Slice only works on low-dimensional embeddings and does not scale to high-dimensional bag-of-words features used in this paper. Some extreme classifiers like PfastreXML~\citep{Jain16} and LEML~\citep{Yu14} can be easily adapted to learn from any relevance weights, but they tend to be inaccurate and inefficient since they train a large ensemble of weak trees and inaccurate low-dimensional projections with linear reconstruction time, respectively.

Performance of extreme classifiers has traditionally been measured in terms of Precision@$k$ and nDCG@$k$~\citep{Prabhu14,Bhatia15}. Recently, propensity-scored metrics were introduced in~\citep{Jain16} which give higher importance to more useful and informative tail labels. However, all these metrics ignore the regression error in the predicted relevance estimates when applied to XR.

{\bf Regression \& ranking}: Most of the conventional regression approaches~\citep{Smola04,Bekkerman11,Engel88,Watson67,Yu14} learn a separate regressor for each outcome variable and hence do not scale to millions of labels. This problem is mitigated to some extent in the multi-objective decision tree based approaches~\citep{Kocev07,Agrawal13,Jain16} which scale sublinearly in the number on outcome variables. However, these approaches suffer from low accuracy issues despite learning a large ensemble of weak trees. As seen from experiments, \alg can be significantly more scalable and accurate than the naive 1-vs-all least squares regressor~\citep{Watson67}, the more efficient LEML regressor with low-rank assumption on the parameter space~\citep{Yu14} and the decision tree based PfastreXML~\citep{Jain16}. The performance accuracy in regression have traditionally been measured by error metrics such as Mean Absolute Deviation (MAD) and Root Mean Square Error (RMSE)~\citep{Chai14}, but these are not appropriate for XR.

Learning to rank methods~\citep{Dhanjal15, Wu18, Park15, Hu18, Yin16, Pradel12, Lee14, Cheng18, Lee14b, Sculley09} have been widely used in the recommendation and ranking literatures, primarily to re-rank a small shortlist of items which has been generated by simple heuristics like tf-idf scoring or by more scalable approaches like extreme classifiers or \alg. These rankers usually have super-linear dependence on the number of labels and hence do not scale to XR. Although negative label sampling could potentially be used to make these approaches more scalable, their ranking performance suffers significantly as demonstrated in Section~\ref{sec:experiments} for the popular RankSVM~\citep{Herbrich99, Lee14b} and the more recent eXtreme Learning to Rank (XLR)~\citep{Cheng18} approaches. A plethora of accuracy metrics have been proposed in the ranking literature~\citep{Prabhu14, Bhatia15, Jarvelin02, Wang13, Kendall38, Zhu04, Lee14}, but none of these measure the regression performance. 

{\bf Dynamic search advertising}: Various approaches have been proposed for DSA in the organic search literature including information retrieval based methods~\citep{Jones00}, probabilistic methods and topic models~\citep{Wei06} and deep learning~\citep{Huang13,Shen14}; however these do not work well for pithy ad-landing pages. Techniques based on landing page summarization~\citep{Choi10}, translation and query language models~\citep{Ravi10,Yih06} and keyword suggestion based on Wikipedia concepts~\citep{Zhang12} have also been proposed for sponsored search; but these suffer from low coverage problem. Extreme classifiers such as Parabel have also been used in DSA to improve accuracy and ad coverage; but they still suffer from low query coverage due to pointwise predictions. As demonstrated in Section~\ref{sec:experiments}, \alg significantly improves query coverage when included in the Bing DSA ensemble comprising all the above alternatives.

\section{Extreme Regression Metrics}
\label{sec:metrics}
{\bf Notation}: Let an XR dataset comprise $N$ data points $\{(\v x_i,\v y_i)\}_{i=1}^N$ where $\v x_i \in \mathbb{R}^D$ is a $D$ dimensional feature vector and $\v y_i \in [0,\infty)^L$ is a ground truth relevance weight vector for point $i$. The weight $y_{il}$ measures the true degree of relevance of label $l$ to point $i$, with higher values indicating higher relevance. Similarly, let $\hat{\v y}_i \in [0,\infty)^L$ denote the predicted relevance weight vector for point $i$. The function $S(\v v,k)$ indicates the {\it ordered} index set of the $k$ highest scoring labels in a score vector $\v v\in [0,\infty)^L$.

{\bf Regression \& ranking metrics}: The regression metrics such as MAD and RMSE; the ranking metrics such as relevance-weighted Precision and nDCG at $k$ (WP@$k$, WN@$k$) and Kendall's Tau at $k$ (Tau@$k$)~\citep{Kendall38}; and WP@$k$-regret which is the difference between the optimal and the attained WP@$k$ are pertinant for this paper. Their formulae are provided in the~\suppurl. The WP@$k$ metric reduces to PSP@$k$, CTR@$k$ suitable for DSA, or Rating@$k$ based on whether $\v y_i$ are set to inverse propensity-scored relevances, ad click-through rates, or user ratings respectively. Rating@$k$ is the undiscounted version of the familiar rating-based nDCG@$k$ metric used in recommender systems~\citep{Jarvelin02}.

{\bf Extreme regression metrics}: Let, $\v e_i$ be the vector of regression errors where $e_{il} = |\hat{y}_{il}-y_{il}|$. The new XR metrics, eXtreme Mean Absolute Deviation at $k$ (XMAD@$k$) and eXtreme Root Mean Square Error at $k$ (XRMSE@$k$) are defined as follows:
\begin{align}
\text{XMAD@}k(\hat{\v y}_i,\v y_i) &= \frac{1}{k}\sum_{l \in S(\v e_i,k)} e_{il}\\
\text{XRMSE@}k(\hat{\v y}_i,\v y_i) &= \sqrt{\frac{1}{k}\sum_{l \in S(\v e_i,k)} e_{il}^2}
\end{align}

For ease of discussion, this paper mainly focusses on the XMAD metric, although most of the observations and results also apply to XRMSE. XMAD@$k$ averages the $k$ maximum regression errors but is minimized when all the $L$ label relevances are predicted exactly right. The following lemma shows that XMAD serves as a good proxy for the ranking error. This is based on an intuition that the ranking errors at the top occur mainly due to either highly underestimating or highly overestimating the relevances of the most or the least relevant labels respectively leading to high regression errors on such labels.
\begin{lemma}
For any true and predicted relevance vectors $\v y,\hat{\v y} \in [0,\infty)^L$, $0 \le \text{WP-regret@}k(\hat{\v y},\v y) \le 2*\text{XMAD@}2k(\hat{\v y},\v y)$ holds true.
\end{lemma}
In addition, $0.5*\text{XMAD@}2k(\hat{\v y},\v y) \le \text{WP-regret@}k(\hat{\v y},\v y)$ also usually holds empirically (see Section~\ref{sec:experiments}) thus making XMAD error a close bound for the ranking error.

Although the top ranked labels with the highest predicted relevances could directly be recommended to a test point, it usually helps to further improve the recommendations by either filtering or re-ranking. The objective of filtering step is to maximize both precision and recall by removing as many irrelevant labels across as many test points as possible. This is crucial in DSA where there are system limitations against online hosting of too many relevant query, ad pairs. The following lemma shows that when the estimated label relevances are almost accurate in terms of the XMAD metric, almost ideal precision-recall trade-offs could be obtained by directly using a global threshold on the predicted relevances.
\begin{lemma}
Given a test set where the true and predicted relevance vectors of $i$th point are $\v y_i,\hat{\v y}_i \in [0,\infty)^L$, $AUPRC \ge AUPRC^* - O(k*XMAD@k)$ holds true where $AUPRC, AUPRC^*$ are the attained and ideal areas under the micro-averaged precision-recall curves plotted using a global threshold.
\end{lemma}
The lemma assumes that the number of retained labels per each test point is less than $k$ for the evaluated region of the curve. It is reasonable to set $k=\log(L)$ since only a small number of labels need to be recommended to each point.

Re-ranking the relevance estimates could significantly improve the final ranking quality, especially when the XMAD errors are small. An example of re-ranking is to combine these estimates with the scores from tail classifiers (see ~\citep{Jain16}) to improve the recommendation accuracies over rare labels. It is worth noting that bad relevance estimates, despite inducing a good initial ranking, could hurt the subsequent filtering or re-ranking performance. Unlike XMAD, the traditional MAD metric is sensitive to the sparsity in the $\hat{\v y}$ vector which does not directly affect the ranking performance in any way. For example, MAD error becomes huge for a dense estimator like 1-vs-All least squares regressor since small regression errors could accrue over million labels into a large value. Results from Section~\ref{sec:experiments} corroborate these observations. 

{\bf Labelwise metrics}: To evaluate performance in the labelwise prediction scenario, all the above ranking and regression metrics, defined for pointwise predictions, need to be redefined appropriately. The formulae for labelwise metrics are provided in the ~\suppurl. Most discussions and results in this paper, while presented primarily for pointwise prediction case, also hold for labelwise prediction setting after interchanging the roles of data points and labels. To promote clarity, all pointwise and labelwise metrics will be used with suffixes "-p" and "-l" respectively.

Note that proofs for the lemmas in this section are available in the ~\suppurl.

\section{\alg: eXtreme Regressor}
This section describes \alg's key components including the label tree construction, the probabilistic regression model and the pointwise and labelwise prediction algorithms using the same model.
\vspace{-4mm}
\subsection{Label Tree Construction}
\alg learns a small ensemble of up to 3 label trees quite similarly to Parabel. Each tree is grown by recursively partitioning the labels into two balanced groups. Label partitioning is achieved by a balanced spherical $k=2$-means algorithm~\citep{Prabhu18b} is which takes as input the feature vectors for all those labels in the current node and outputs 2 label clusters, efficiently, in $O(\hat{D}L\log L)$ time where $\hat{D}$ is the number of non-zero features per data point. The feature vector for a label is represented by the unit vector that points along the average of the training points which are relevant to the label:
\begin{align}
    \label{eqn:lbl_ft}
     \v v_l = \v v'_l/\|\v v'_l\|_2\mypos{5}\mbox{where}\mypos{5}\v v'_l = \sum_{i=1}^N y_{il}\v x_i
\end{align}
This is based on the intuition that two labels are similar if they are active in similar training points. In DSA, two queries (labels) are similar according to the proposed representation if they lead to clicks on similar ads (training points). As a result, the $k$-means algorithm ensures that the labels relevant for a data point end up in the same leaf. Note that, unlike Parabel, \alg uses non-binary relevance-weighted average leading to more informative label feature representations.
\vspace{-1mm}
\subsection{A Probabilistic Regression Model}
\alg is a regression method which takes a probabilistic approach to estimating the label relevance weights. Firstly, all the relevance weights are normalized to lie between $0$ and $1$ by dividing by its maximum value, thus allowing them to be treated as probability values. Note that while click-through rates in DSA are already valid probabilities, the inverse propensities and the user rating could exceed $1$. Also, note that the predicted estimates can be easily scaled back since no information is lost due to this normalization.

\alg treats the normalized relevance weights for each label as the marginal probability of its relevance to a data point, which is, in fact, the case in DSA. This allows \alg to minimize the KL-divergence between the true and the predicted marginal probability for each label with respect to each data point. KL-divergence~\citep{Kullback51} measures how close 2 distributions are and is minimized when the 2 are identical, thus justifying its use while regressing on to probability values. 

A naive 1-vs-All approach, which learns a separate regressor minimizing KL-divergences for each label, would be extremely costly to train when labels are in millions. To reduce this complexity, \alg leverages the previously trained label tree. \alg expresses the marginal probability of a label as the probability that a data point traverses the tree path starting from the root to the label. Let the path from root to label $l$ consist of nodes $n_{l1},\cdots,n_{lH}$ where $H$ is tree height, $n_{l1}$ is the root and $n_{lH}$ is the leaf node  containing solely label $l$. Let $z_{lh}$ denote the probability that a data point $\v x$ visits the node $n_{lh}$ after it has already visited the parent $n_{l(h-1)}$. Then the true marginal probability $y_l$ that the label $l$ is relevant to $\v x$ is equivalent to $y_l = \prod_{h=1}^{H} z_{lh}$. Similar equality holds for predicted marginal probability: $\hat{y}_l = \prod_{h=1}^{H} \hat{z}_{lh}$. \alg then learns to minimize an upper bound on the KL-divergence between the two according to the following theorem.

\begin{theorem}
Given that $y_l = \prod_{h=1}^{H} z_{lh}$ and $\hat{y}_l = \prod_{h=1}^{H} \hat{z}_{lh}$ and under the standard unvisited node assumption of Parabel
\begin{align}
\label{eqn:kld}
    D_{KL} (y_l || \hat{y}_l)
    \le \sum_{h=1}^H s_{lh} D_{KL}( z_{lh} || \hat{z}_{lh} )~~~\text{where}~~~ s_{lh}=\prod_{\tilde{h}=1}^h z_{l(\tilde{h}-1)}
\end{align}
\end{theorem}
\begin{proof}
Proof is provided in the \suppurl.\vspace{-2mm}
\end{proof}
The unvisited node assumption formalizes the observation that the children of an unvisited internal node will never be traversed and that the labels in an unvisited leaf node will never be visited by a data point~\citep{Prabhu18b}.
Due to the above theorem, \alg can separately minimize the KL-divergence over the true and predicted probabilities that a data point takes a particular edge in the tree, and still end up minimizing the KL-divergences over each of the individual marginal label probabilities. The true probability value of edge traversal $z_{lh}$ is essentially the probability that the data point visits any of the labels in the subtree rooted at the node indexed $lh$. We instantiate it to be equal to the largest marginal probability of any label in the subtree, by assuming the worst-case scenario that labels in each subtree are fully correlated, which promotes model robustness. 

The KL-divergence minimization is mathematically equivalent to training a logistic regressor for estimating $z_{lh}$ values for each tree edge where every data point is duplicated with weights $z_{lh}$ and $1-z_{lh}$:
\begin{align}
\label{eqn:obj}
    \min_{\v w_{n}} \|\v w_{n}\|^2 + \frac{C}{|\mathcal{I}_n|}\sum_{i\in \mathcal{I}_{n}}  \{s_{in}z_{in} \log(1 + \exp(-\v w_{n}^{\top}x_i)) +\\ s_{in}(1-z_{in})\log(1 + \exp(+\v w_{n}^{\top}x_i))\}
\end{align}
where $n$ is used to index the node instead of $lh$, $\mathcal{I}_{n}$ only include those points which reach the node $n$. The problem in (Eq.~\ref{eqn:obj}) is strongly convex and was optimized using the modified CDDual algorithm available from Liblinear package~\citep{Fan08}. To summarize, each internal node in \alg contains 2 1-vs-All regressors which give the probability that a data point traverses to each of its children, each leaf node contains $M$ 1-vs-All regressors which gives the conditional probability of each label being relevant given the data point reaches its leaf. 

We make a mild assumption that each data point has at most $O(\log L)$ positive labels is made which is often valid on extreme learning datasets. As a result, each data point traverses at most $O(\log^2 L)$ tree edges, which directly leads to a huge reduction in training complexity thus resulting in $O(N\hat{D}\log^2 L)$ where $\hat{D}$ is the average number of non-zero features per data point. The following lemma describes how \alg's training objective is related to the XMAD@$k$ metric proposed earlier:
\begin{lemma}
\alg's overall training objective minimizes an upper bound over XMAD@$k$ for all $k$, with the bound being tighter for smaller $k$ values.
\end{lemma}
\begin{proof}
The proof is provided in the \suppurl.
\end{proof}

\subsection{Pointwise Inference}
The pointwise inference algorithm in \alg utilizes the same beam search prediction technique proposed in Parabel where only the top ranked relevant labels are recommended based on a greedy, breadth-first tree traversal strategy. The following theorem proves that such traversal mechanism is not only asymptotically optimal for both WP@$k$ and XMAD@$k$ but also strongly generalizable with $O(polylog(L))$ sample complexity. This uses the assumption that each data point has at most $O(\log L)$ positive labels. Also the theorem assumes that each individual regressor in well-generalizable and achieves zero-regret with infinite data samples.

\begin{theorem}
When each data point has at most $O(\log L)$ positive labels, the expected WP@$k$ regret and XMAD@$k$ error suffered by \alg's pointwise inference algorithm are bounded by:\\
$$O(\log^2 L\sqrt{\frac{W}{\sqrt{Np}}}\sqrt{1+\sqrt{5\log{\frac{3L}{\delta}}} })$$ \\with probability at least $1-\delta$, where $N$ is the total training points, $L$ is the number of labels, $W$ is the maximum norm across all node classifier vectors and $p$ is the minimum probability density of $\v x$ distribution that any tree node receives. 
\end{theorem}\vspace{-1mm}
Proof is available in the \suppurl.
Therefore the errors go to $0$ as $N \to \infty$. The $\log^2 L$ dependence arises because each data point visits at most $\log^2 L$ nodes in a tree.
\vspace{-1mm}
\subsection{Labelwise Inference}
The \alg model also allows efficient labelwise inference. The core idea here is to estimate from training data the fraction of points with non-zero relevance that visit each node of the tree and allot a factor $F$ times the same fraction of top ranking test points to respective nodes. On large scale datasets with enough training and test points, the ratio of non-zero relevance points in each tree node remain almost the same over training and test points. The factor accounts for any small deviations. This strategy is adopted to ensure that all non-zero relevance points for a label end up reaching the label's leaf node. Finally, the topmost scoring test points that visit a label's leaf node are ranked at the top for that label, where the scores are marginal relevance probabilities, the average test time complexity is $O(F\log^2 L)$ per test point. Pseudocode for labelwise inference is provided in the \suppurl.

\begin{table}[b!]
\vspace{-3mm}
\centering
\captionsetup{font=small}
\caption{Dataset statistics}
\label{tab:stats}
\resizebox{\linewidth}{!}
	{
		\begin{tabular}{lrrrrcc}
			\toprule
			Dataset	& \multicolumn{1}{c}{Train}		& \multicolumn{1}{c}{Features}	& \multicolumn{1}{c}{Labels}	& \multicolumn{1}{c}{Test}	& \multicolumn{1}{c}{Avg. labels}	& \multicolumn{1}{c}{Avg. points}	\\
															& \multicolumn{1}{c}{$N$}		& \multicolumn{1}{c}{$D$}		& \multicolumn{1}{c}{$L$}		& \multicolumn{1}{c}{$N'$}   & \multicolumn{1}{c}{per point}     & \multicolumn{1}{c}{per label}     \\
			\midrule
		     BibTeX		& 4,880		& 1,836	& 159		& 2,515		& \phantom{00}2.40		& 111.71	\\
			EURLex-4K		& 15,539		& 5,000		& 3,993		& 3,809		& \phantom{00}5.31		& 448.57	\\
			Wiki10-31K  & 14,146		& 101,938 & 30,938	& 6,616	& \phantom{0}18.64     & \phantom{00}8.52		\\
			SSA-130K			& 122,462	& 152,192	& 130,515	& 54,773	& \phantom{00}5.60		& \phantom{00}7.60	\\
			WikiLSHTC-325K  & 1,778,351		& 1,617,899 & 325,056	& 587,084	& \phantom{00}3.26		& \phantom{0}23.74		\\
			Amazon-670K		& 490,449		& 135,909	& 670,091	& 153,025	& \phantom{00}5.38		& \phantom{00}5.17	\\\midrule
			YahooMovie-8K		& 8,341	& 28,978	& 7,642	& 3,574	& \phantom{0}18.57		& \phantom{0}28.96		\\	
			DSA-130K			& 122,462	& 152,192	& 130,515	& 54,773	& \phantom{00}5.60		& \phantom{00}7.60	\\
			MovieLens-138K			& 18,732	& 19,924	& 138,490	& 8,012	& 527.31	& 101.83	\\
			DSA-1M & 4,914,640 & 1,840,877 & 1,453,150 & 2,106,273 & \phantom{00}0.23 & \phantom{00}7.80\\
			\bottomrule
		\end{tabular}
	}
	\end{table}


\begin{table*}[t!]
\centering
\captionsetup{font=small}
	\caption{\alg achieves the best or close to the best ranking and regression performance in both pointwise ("-p") and labelwise ("-l") prediction settings. Re-ranking with tail classifiers (\alg-t) further improves the performance in many cases. More results are in the \suppurl.}
	\label{tab:res}
	\begin{minipage}[t]{0.5\linewidth}
		\resizebox{0.95\linewidth}{!}
		{
		\begin{tabular}{@{}lcccccc@{}}
\toprule
Method     & \begin{tabular}[c]{@{}c@{}}PSP-p@5 \\ (\%)\end{tabular} & \begin{tabular}[c]{@{}c@{}}Tau-p@5 \\ (\%)\end{tabular} & \begin{tabular}[c]{@{}c@{}}XMAD-p@5 \\ \end{tabular} & \begin{tabular}[c]{@{}c@{}}Training \\ time (hrs)\end{tabular} & \begin{tabular}[c]{@{}c@{}}Test time\\ /point (ms)\end{tabular} & \begin{tabular}[c]{@{}c@{}}Model \\ size (GB)\end{tabular} \\
 \midrule
\multicolumn{7}{c}{\textbf{BibTex}}                                                                                                                                                                                                                                                                                                                                                          \\ \midrule
PfastreXML & 59.75                                                 & 53.68                                                 & \textbf{0.3151}                                        & 0.0050                                                         & 0.2348                                                          & 0.0246                                                     \\
Parabel    & 57.36                                                 & 51.48                                                 & 0.3372                                                 & 0.0015                                                         & 0.1945                                                          & 0.0035                                                     \\
LEML       & 56.42                                                 & 51.58                                                 & 0.3520                                                 & 0.0229                                                         & 0.1737                                                          & 0.0032                                                     \\
1-vs-all-LS        & \textbf{60.14}                                        & \textbf{54.21}                                        & 0.3337                                                 & \textbf{0.0007}                                                & 0.1137                                                          & 0.0023                                                     \\
RankSVM    & 59.12                                                 & 52.58                                                 & 0.7089                                                 & 0.0015                                                         & \textbf{0.0719}                                                 & 0.0023                                                     \\
DiSMEC     & 57.23                                                 & 51.47                                                 & 0.3371                                                 & 0.0004                                                         & 0.0951                                                          & \textbf{0.0012}                                            \\
ProXML     &   58.30                                                    &   -                                                    &   -                                                     &    -                                                            &   -                                                              &   -                                                         \\
\alg        & 58.61                                                 & 52.35                                                 & 0.3158                                                 & 0.0035                                                         & 0.1642                                                          & 0.0030                                                     \\
\alg-t      & 58.77                                                 & 52.46                                                 & 0.3386                                                 & 0.0025                                                         & 0.1256                                                          & 0.0043                                                     \\ 
\midrule
\multicolumn{7}{c}{\textbf{EURLex-4K}}                                                                                                                                                                                                                                                                                                                                                       \\ \midrule
PfastreXML & 45.17                                                 & 48.85                                                 & 0.1900                                                 & 0.0887                                                         & 1.3891                                                          & 0.2265                                                     \\
Parabel    & 48.29                                                 & 50.75                                                 & 0.4227                                                 & \textbf{0.0245}                                                & \textbf{1.1815}                                                 & 0.0258                                                     \\
LEML       & 32.30                                                 & 37.24                                                 & 0.2115                                                 & 0.3592                                                         & 4.4483                                                          & 0.0281                                                     \\
1-vs-all-LS        & \textbf{52.27}                                        & \textbf{53.96}                                        & \textbf{0.1744}                                        & 0.1530                                                         & 4.5378                                                          & 0.1515                                                     \\
RankSVM    & 46.70                                                 & 51.43                                                 & 1.1967                                                 & 0.1834                                                         & 4.7635                                                          & 0.1470                                                     \\
DiSMEC     & 50.62                                                 & 52.33                                                 & 0.4308                                                 & 0.0999                                                         & 1.9489                                                          & \textbf{0.0072}                                            \\
ProXML     &    51.00                                                   &    -                                                   &        -                                                &           -                                                     &      -                                                           & -                                                           \\
\alg        & 49.72                                                 & 52.86                                                 & 0.1849                                                 & 0.0642                                                         & 1.2899                                                          & 0.0378                                                     \\
\alg-t      & 50.40                                                 & 53.45                                                 & 0.2132                                                 & 0.0544                                                         & 1.2074                                                          & 0.0692                                                     \\ \midrule
\multicolumn{7}{c}{\textbf{Wiki10-31K}}                                                                                                                                                                                                                                                                                                                                                      \\ \midrule
PfastreXML & 15.91                                                 & 20.29                                                 & 0.5705                                                 & 0.3491                                                         & 11.6855                                                         & 0.5466                                                     \\
Parabel    & 13.68                                                 & 19.83                                                 & 0.7085                                                 & \textbf{0.3204}                                                & \textbf{3.7275}                                                 & 0.1799                                                     \\
LEML       & 13.05                                                 & 20.06                                                 & 0.5716                                                 & 0.9546                                                         & 54.9470                                                         & 0.5275                                                     \\
1-vs-all-LS        & 21.89                                                 & 26.71                                                 & \textbf{0.5459}                                        & 2.4341                                                         & 129.8342                                                        & 16.9871                                                    \\
RankSVM    & 18.46                                                 & 25.84                                                 & 1.2236                                                 & 4.9631                                                         & 92.2684                                                         & 10.8536                                                    \\
DiSMEC     & 15.61                                                 & 22.43                                                 & 0.7140                                                 & 2.1945                                                         & 13.8993                                                         & \textbf{0.0290}                                            \\
\alg        & 16.94                                                 & 24.97                                                 & 0.5716                                                 & 0.6184                                                         & 3.7649                                                          & 0.3218                                                     \\
\alg-t      & \textbf{22.60}                                        & \textbf{30.55}                                        & 0.5506                                                 & 0.6431                                                         & 5.4910                                                          & 0.9026                                                     \\\midrule \multicolumn{7}{c}{\textbf{WikiLSHTC-325K}}                                                                                                                                                                                                                                                                                                                                                  \\ \midrule
PfastreXML & 28.04                                                 & 36.38                                                 & 0.1437                                                 & 7.1974                                                         & 6.9045                                                          & 13.3096                                                    \\
Parabel    & 37.22                                                 & 41.71                                                 & 0.2459                                                 & \textbf{1.2195}                                                & \textbf{2.2486}                                                 & \textbf{3.0885}                                            \\
DiSMEC    & 39.50                                                & -                                                 & -                                                 & -                                              & -                                                 & -                                          \\
ProXML    & \textbf{41.00}      & -                                                 & -                                                 & -                                              & -                                                 & -                                            \\
\alg        & 36.92                                                 & 41.62                                                 & \textbf{0.1411}                                        & 4.5119                                                         & 3.0312                                                          & 3.5105                                                     \\
\alg-t      & 40.33                                        & \textbf{43.39}                                        & 0.3140                                                 & 3.8552                                                         & 3.0896                                                          & 4.1955                                                     \\ \midrule
\multicolumn{7}{c}{\textbf{Amazon-670K}}                                                                                                                                                                                                                                                                                                                                                     \\ \midrule
PfastreXML & 28.53                                                 & 30.97                                                 & 0.4019                                                 & 3.3143                                                         & 11.4931                                                         & 9.8113                                                     \\
Parabel    & 32.88                                                 & 31.32                                                 & 0.4292                                                 & \textbf{0.5815}                                                & 2.3419                                                          & \textbf{1.9297}                                            \\
DiSMEC     & 34.45                                     & 31.94                                                 & 0.4275                                                 & 373                                                            & 1414                                                            & 3.7500                                                     \\
ProXML    & \textbf{35.10}      & -                                                 & -                                                 & $\approx$1200                                              & $\approx$1000                                                & -                                            \\
\alg        & 33.24                                                 & 34.72                                                 & \textbf{0.3869}                                        & 1.4925                                                         & 2.4633                                                          & 3.4186                                                     \\
\alg-t      & 34.29                                                 & \textbf{35.83}                                        & 0.4473                                                 & 1.1864                                                         & \textbf{2.2242}                                                 & 4.5952                                                     \\\bottomrule
\end{tabular}
		}
		\end{minipage}%
		\begin{minipage}[t]{0.5\linewidth}
		\resizebox{0.95\linewidth}{!}
		{
		\begin{tabular}{@{}lcccccc@{}}\toprule
Method     & \begin{tabular}[c]{@{}c@{}}CTR-p@5 \\ (\%)\end{tabular}   & \begin{tabular}[c]{@{}c@{}}Tau-p@5 \\ (\%)\end{tabular} & \begin{tabular}[c]{@{}c@{}}XMAD-p@5 \\ \end{tabular} & \begin{tabular}[c]{@{}c@{}}Training \\time (hrs)\end{tabular} & \begin{tabular}[c]{@{}c@{}}Test time\\ /point (ms)\end{tabular} & \begin{tabular}[c]{@{}c@{}}Model \\ size (GB)\end{tabular} \\ \midrule
\multicolumn{7}{c}{\textbf{SSA-130K}}                                                                                                                                                                                                                                                                                                                                                        \\ \midrule
PfastreXML & 27.79                                                 & 23.77                                                 & 0.0655                                                 & 1.3765                                                         & 5.2419                                                          & 1.6258                                                     \\
Parabel    & \textbf{32.97}                                        & \textbf{30.25}                                        & 0.1430                                                 & \textbf{0.2283}                                                & 1.9098                                                          & \textbf{0.3625}                                            \\
LEML       & 6.54                                                  & 8.10                                                  & \textbf{0.0654}                                        & 8.3253                                                         & 161.6891                                                        & 1.1308                                                     \\
RankSVM & 13.06 & 14.03 & 2.7871 & 9.6026 & 130.0945 & 7.4834\\
DiSMEC & 32.75 & 29.16 & 0.1562 & 31.4358 & 61.0967 & 0.0802\\
\alg        & 32.39                                                 & 28.27                                                 & 0.0684                                                 & 0.4570                                                         & 7.4715                                                          & 0.7871                                                     \\
\alg-t      & 32.81                                                 & 28.73                                                 & 0.1131                                                 & 0.5049                                                         & \textbf{1.7746}                                                 & 1.4156                                                     \\

\midrule\midrule
Method     & \begin{tabular}[c]{@{}c@{}}Rating-l@5 \\ (\%)\end{tabular}   & \begin{tabular}[c]{@{}c@{}}Tau-l@5 \\ (\%)\end{tabular} & \begin{tabular}[c]{@{}c@{}}XMAD-l@5 \\ \end{tabular} & \begin{tabular}[c]{@{}c@{}}Training \\time  (hrs)\end{tabular} & \begin{tabular}[c]{@{}c@{}}Test time\\ /point (ms)\end{tabular} & \begin{tabular}[c]{@{}c@{}}Model \\ size (GB)\end{tabular} \\ \midrule
\multicolumn{7}{c}{\textbf{YahooMovie-8K}}                                                                                                                                                                                                                                                                                                                                                   \\ \midrule
PfastreXML & 10.18                                                 & 19.72                                                 & 0.6286                                                 & \textbf{0.0241}                                                & 8.5074                                                          & 0.0753                                                     \\
Parabel    & 9.73                                                  & 28.22                                                 & 0.6284                                                 & 0.0299                                                         & \textbf{0.9639}                                                 & 0.1307                                                     \\
LEML       & 21.79                                                 & 28.85                                                 & 0.6408                                                 & 0.0593                                                         & 5.3650                                                          & 0.0586                                                     \\
1-vs-all-LS        & 21.63                                                 & 31.24                                                 & 0.6269                                                 & 0.0740                                                         & 6.8841                                                          & 1.6977                                                     \\
RankSVM    & 24.88                                                 & 33.28                                                 & 1.0579                                                 & 0.1282                                                         & 5.1620                                                          & 0.7172                                                     \\
DiSMEC     & 24.53                                                 & 32.75                                                 & 0.6207                                                 & 0.0337                                                         & 3.4258                                                          & 0.0376                                            \\
XLR       & 4.66                                               & 10.72                                                & 0.6716                                        & -                                                         & 4.7724                                                       & \textbf{0.0293}                                                    \\
\alg        & 25.86                                                 & 35.00                                                 & 0.6248                                                 & 0.0685                                                         & 4.1965                                                          & 0.2829                                                     \\
\alg-t      & \textbf{26.05}                                        & \textbf{35.33}                                        & \textbf{0.6185}                                        & 0.0615                                                         & 3.6353                                                          & 0.4500                                                     \\ \midrule
\multicolumn{7}{c}{\textbf{MovieLens-138K}}                                                                                                                                                                                                                                                                                                                                         \\ \midrule
PfastreXML & 7.25                                                  & 22.84                                                 & 0.9199                                                 & 0.4514                                                         & 19.8270                                                         & 0.1837                                            \\
Parabel    & 3.51                                                  & 37.80                                                 & 0.9200                                                 & 1.7790                                                         & \textbf{1.6132}                                                 & 3.4322                                                     \\
LEML       & 43.19                                                 & 64.78                                                 & 0.8722                                                 & \textbf{0.4186}                                                & 91.4262                                                         & 0.2535                                                     \\
1-vs-all-LS        & 42.16                                                 & 63.92                                                 & 0.8832                                                 & 2.5756                                                         & 121.6169                                                        & 16.1334                                                    \\
DiSMEC     & 45.35                                                 & 61.55                                                 & 0.8857                                                 & 1.5437                                                         & 74.9537                                                         & 1.0514                                                     \\
XLR       & 9.67                                               & 21.42                                                & 0.9134                                        & 4.579                                                         & 68.347                                                       &\textbf{0.0634}                                                    \\
\alg        & 48.94                                                 & 66.99                                                 & 0.8741                                                 & 2.6287                                                         & 7.7996                                                          & 3.6223                                                     \\
\alg-t      & \textbf{49.29}                                        & \textbf{67.36}                                        & \textbf{0.8285}                                        & 2.7437                                                         & 9.8279                                                          & 4.8958                                                     \\ \midrule\midrule
Method     & \begin{tabular}[c]{@{}c@{}}CTR-l@5 \\ (\%)\end{tabular}   & \begin{tabular}[c]{@{}c@{}}Tau-l@5 \\ (\%)\end{tabular} & \begin{tabular}[c]{@{}c@{}}XMAD-l@5 \\ \end{tabular} & \begin{tabular}[c]{@{}c@{}}Training \\time  (hrs)\end{tabular} & \begin{tabular}[c]{@{}c@{}}Test time\\ /point (ms)\end{tabular} & \begin{tabular}[c]{@{}c@{}}Model \\ size (GB)\end{tabular} \\ \midrule
\multicolumn{7}{c}{\textbf{DSA-130K}}                                                                                                                                                                                                                                                                                                                                                        \\ \midrule
PfastreXML & 28.18                                                 & \textbf{34.75}                                        & 0.0422                                                 & 1.3765                                                         & 5.2419                                                          & 1.6258                                                     \\
Parabel    & 33.97                                                 & 28.37                                                 & 0.0891                                                 & \textbf{0.2283}                                                & \textbf{1.9098}                                                 & 0.3625                                            \\
LEML       & 10.36                                                 & 7.70                                                  & \textbf{0.0415}                                        & 8.3253                                                         & 212.1707                                                        & 1.1308                                                     \\
DiSMEC & 34.06 & 27.96 & 0.1039 & 31.4358 & 55.4037 & 0.0802\\
XLR       & 0.09                                               & 0.10                                                & 0.4816                                        &    5.5430                                                      & 64.1134                                                       &\textbf{0.0678}                                                    \\
\alg        & 35.66                                                 & 28.51                                                 & 0.0439                                                 & 0.4570                                                         & 7.4715                                                          & 0.7871                                                     \\
\alg-t      & \textbf{36.32}                                        & 28.45                                                 & 0.0587                                                 & 0.3669                                                         & 8.4376                                                          & 1.3512                                                     \\  \midrule
\multicolumn{7}{c}{\textbf{DSA-1M}} \\ \midrule
Parabel & 37.95 & 30.93 & 0.1004 & \textbf{9.2800} & \textbf{2.5031}& \textbf{5.6774}\\ 
\alg & 37.57 & 31.09 & \textbf{0.0563} &20.7463 &3.1792&11.0178\\
\alg-t &\textbf{38.81} & \textbf{31.41} & 0.0714  &  15.4201  &  3.4036 &   18.7434\\\bottomrule
\end{tabular}
		}
		\end{minipage}
\end{table*}
\vspace{-1mm}
\section{Experiments}
\label{sec:experiments}
{\bf Datasets}: Experiments were carried out on several medium and large scale benchmark datasets with up to 4.9M training points, 1.8M features and 1.4M labels (see Table~\ref{tab:stats} for dataset statistics). These datasets cover diverse applications such as document tagging (BibTeX~\citep{Prabhu14}, EURLex-4K~\citep{Mencia08}, Wiki10-31K~\citep{XMLcode} \& WikiLSHTC-325K~\citep{Bhatia15, Partalas15}), content-based movie recommendation (YahooMovie-8K~\citep{YahooMovies} \& MovieLens-138K~\citep{Harper15, MovieLens}), item-to-item recommendation of Amazon products (Amazon-670K~\citep{Bhatia15, McAuley13}), sponsored search advertising (SSA-130K) and dynamic search advertising (DSA-130K, DSA-1M). For ease of discussion, the label size suffixes are dropped from dataset names hereafter except for DSA. The document tagging, item-to-item recommendation, and SSA datasets require pointwise inference whereas the movie recommendation and DSA datasets require labelwise inference. YahooMovie and MovieLens use normalized (between 0 and 1) user-provided movie ratings as relevance weights and movie meta-data like summary, genres, and tags as features. For all the datasets, bag-of-words feature representation derived from text descriptions are used. SSA and DSA are proprietary datasets that were created by mining the Bing logs. Rest of the datasets are available from~\citep{XRegCode}.

{\bf Baselines}: \alg was compared to leading extreme classifiers such as PfastreXML~\citep{Jain16}, Parabel~\citep{Prabhu18b}, DiSMEC~\citep{Babbar17} and ProXML~\citep{Babbar18}, traditional multivariate regressors such as one-vs-all least-squares regression (1-vs-all-LS) and LEML~\citep{Yu14}, and a popular pairwise ranker, RankSVM~\citep{Lee14b, Herbrich99}. \alg was also compared to the recent eXtreme Learning to Rank (XLR)~\citep{Cheng18} approach. ProXML is the current state-of-the-art over propensity scored precision@$k$ (PSP-p@$k$) during pointwise inference. Results for DiSMEC and ProXML, which required 1000s of cores, could not be replicated on large datasets and hence the numbers from the corresponding papers have been reported directly. RankSVM was unable to scale to datasets larger than SSA-130K and hence required down-sampling of negatives up to 0.1\% on these larger datasets. XLR, which specifically addresses the labelwise recommendation task, has only been applied to labelwise datasets. For the other baselines, results have been reported for only those datasets up to which the implementations scale. Since many of these large-scale datasets have a preponderance of tail labels, results for a variant of \alg where predicted labels have been reranked with tail classifier scores have also been reported with a "-t" suffix. The tail classifiers are generative classifiers which are tailored for accurate predictions on labels with $<5$ training point samples~\citep{Jain16}. For extreme classifiers which train on binary labels (Parabel, DiSMEC, and ProXML), all positive relevance weights were approximated to be fully relevant (value $1$). Remaining baselines, including the PfastreXML and LEML, were trained on relevance weighted labels for a fair comparison.

{\bf Hyperparameters}: \alg has 5 hyperparameters: (a) number of label trees in the ensemble ($T$); (b) number of tree paths explored by a test point during pointwise prediction ($P$); (c) maximum ratio of test to train points that traverse to each node during labelwise prediction ($F$); (d) maximum number of labels in a leaf node of \alg tree ($M$); and (e) regularization parameter common to  logistic regressors in all the internal and leaf node classifiers ($C$). On medium-sized datasets, the \alg's hyperparameters were set by fine-grained tuning over a 10\% validation dataset. On larger datasets where tuning was not feasible, the default hyperparameter setting of $T=3$, $P=10$, $F=4$, $M=100$ and $C=10$ was used. Results in table~\ref{tab:hyper} of the \suppurl~ demonstrates that the above default values of $T,P,M$ achieve the best trade-off between accuracy and scalability across multiple datasets and increasing any of them further leads to minimal gains in accuracy while linearly increasing the training or prediction cost. The value of $\alpha$, which adjusts the influence of tail classifiers in \alg-t, was also tuned on the validation data. The hyperparameters for baseline algorithms were also set by tuning on medium datasets and set to defaults suggested in the respective papers/codebases on larger datasets.

{\bf Metrics and hardware}: Performances were evaluated using accuracy metrics such as WP@$k$ variants, Tau@$k$ and XMAD@$k$ (see Section~\ref{sec:metrics}) as well as efficiency metrics such as training time, test time per data point and model size. Among WP@$k$ variants, for tagging (BibTeX, EURLex, Wiki10, WikiLSHTC) and Amazon datasets PSP@$k$ are reported; for SSA and DSA which are ads datasets CTR@$k$ is reported; and for movie recommendation datasets (YahooMovie and MovieLens) Rating@$k$ is reported. All accuracy metrics are suffixed with "-p" or "-l" depending on whether the prediction scenario is pointwise or labelwise. All experiments were run on an Intel Xeon 2.5 GHz processor with 256 GB RAM.


\begin{table*}[t!]
\centering
\captionsetup{font=small}
	\caption{XMAD@$k$ is a better indicator of the filtering and re-ranking qualities than purely ranking metrics like WP@$k$ or traditional regression metrics like MAD.}
	\label{tab:prcurves}
	\begin{minipage}[t]{0.5\linewidth}
		\resizebox{\linewidth}{!}
		{
\begin{tabular}{@{}lccccc@{}}
\toprule
Method    & \multicolumn{1}{l}{AUPRC} & \begin{tabular}[c]{@{}c@{}}WP-rerank-p\\ @5 (\%)\end{tabular} & \multicolumn{1}{l}{XMAD-p@5} & \multicolumn{1}{l}{MAD} & \begin{tabular}[c]{@{}c@{}}WP-p\\ @5 (\%)\end{tabular} \\ \midrule
\multicolumn{6}{c}{EURLex-4K}                                                                                                                                                                                           \\ \midrule
Parabel   & 0.092                     & 49.67                                                         & 0.4227                       & 3.96                  & 48.29                                                  \\
\alg      & \textbf{0.117}            & \textbf{50.39}                                                & \textbf{0.1849}              & 1.22                  & \textbf{49.72}                                         \\
\alg-zero & 0.085                     & 50.12                                                         & 0.2255                       & \textbf{1.21}         & \textbf{49.72}                                         \\ \midrule
\multicolumn{6}{c}{Wiki10-31K}                                                                                                                                                                                          \\ \midrule
Parabel   & 0.036                    & 21.14                                                         & 0.7084                       & 10.15                 & 13.67                                                 \\
\alg      & \textbf{0.046}           & \textbf{22.60}                                                & \textbf{0.5716}              & 6.01                  & \textbf{16.94}                                       \\
\alg-zero & 0.036                    & 19.20                                                         & 0.5781                       & \textbf{5.61}         & \textbf{16.94}                                       \\ \bottomrule
\end{tabular}
		}
		\end{minipage}%
		\begin{minipage}[t]{0.5\linewidth}
		\resizebox{\linewidth}{!}
		{
\begin{tabular}{@{}lccccc@{}}
\toprule
Method    & \multicolumn{1}{l}{AUPRC} & \begin{tabular}[c]{@{}c@{}}WP-rerank-l\\ @5 (\%)\end{tabular} & \multicolumn{1}{l}{XMAD-p@5} & \multicolumn{1}{l}{MAD} & \begin{tabular}[c]{@{}c@{}}WP-l\\ @5 (\%)\end{tabular} \\ \midrule
\multicolumn{6}{c}{YahooMovie-8K}                                                                                                                                                                                       \\ \midrule
Parabel   & 0.135                     & 10.09                                                         & 0.6283                       & 6.39                  & 9.72                                                   \\
\alg      & \textbf{0.175}            & \textbf{26.03}                                                & \textbf{0.6248}              & 6.78                  & \textbf{25.85}                                         \\
\alg-zero & 0.076                     & 25.93                                                         & 0.6306                       & \textbf{5.83}          & \textbf{25.85}                                         \\ \midrule
\multicolumn{6}{c}{DSA-130K}                                                                                                                                                                                            \\ \midrule
Parabel   & 0.016                     & 34.59                                                         & 0.0890                       & 0.88                  & 33.97                                                  \\
\alg      & \textbf{0.035}            & \textbf{36.32}                                                & 0.0438                       & 0.24                  & \textbf{35.65}                                         \\
\alg-zero & 0.010                     & 35.90                                                          & \textbf{0.0402}              & \textbf{0.20}         & \textbf{35.65}                                         \\ \bottomrule
\end{tabular}
		}
		\end{minipage}
\end{table*}

\begin{table*}[]
\centering
\captionsetup{font=small}
	\caption{Ranking regret at k is up to 69x more closely bounded by 2*XMAD@2k compared to the traditional MAD as proposed in Section~\ref{sec:metrics}. $k = 5$, "-p": pointwise, "-l": labelwise and "-t": use of tail classifiers. Please refer to the text for details.}
	\label{tab:regret}
	\begin{minipage}[t]{0.5\linewidth}
		\resizebox{\linewidth}{!}
		{
\begin{tabular}{@{}lccccc@{}}
\toprule
Method     & \begin{tabular}[c]{@{}c@{}}WP-regret-p\\ @k\end{tabular} & \begin{tabular}[c]{@{}c@{}}2*XMAD-p\\ @2k\end{tabular} & MAD             & \begin{tabular}[c]{@{}c@{}}2*XMAD-p@2k / \\ WP-regret-l@k\end{tabular} & \begin{tabular}[c]{@{}c@{}}MAD / \\ WP-regret-p@k\end{tabular} \\ \midrule
\multicolumn{6}{c}{EURLex-4K}                                                                                                                                                                                                                                                            \\\midrule 
PfastreXML & 0.1237                                                   & 0.2481                                                 & 1.8667          & 2.01                                                                   & 15.09                                                           \\
Parabel    & 0.1166                                                   & 0.5696                                                 & 3.9622          & 4.89                                                                   & 33.98                                                           \\
LEML       & 0.1527                                                   & 0.2739                                                 & 2.0779          & \textbf{1.79}                                                          & 13.61                                                           \\
1-vs-all-LS        & 0.1076                                                   & 0.2468                                                 & 2.696           & 2.29                                                                   & 25.06                                                         \\
RankSVM    & 0.1202                                                   & 2.1303                                                 & 37.7822         & 17.72                                                                  & 314.33                                                         \\
DiSMEC     & 0.1107                                                   & 0.5733                                                 & 4.9883          & 5.18                                                                   & 45.06                                                           \\
\alg        & 0.1134                                                   & \textbf{0.2432}                                        & \textbf{1.2284} & 2.14                                                                   & \textbf{10.83}                                                  \\
\alg-t      & 0.1119                                                   & 0.3141                                                 & 3.4887          & 2.81                                                                   & 31.18                                                          \\ \midrule
\multicolumn{6}{c}{Wiki10-31K}                                                                                                                                                                                        \\ \midrule
PfastreXML                 & 0.4861          & 0.8862          & 9.3561          & 1.82                                                                 & 19.25                                                       \\
Parabel                    & 0.4990          & 1.1784          & 10.1523         & 2.36                                                                 & 20.35                                                        \\
LEML                       & 0.5027          & 0.8886          & 25.5622         & \textbf{1.77}                                                        & 50.85                                                       \\
1-vs-all-LS                & 0.4515          & \textbf{0.8492} & 34.1409         & 1.88                                                                 & 75.62                                                       \\
RankSVM                    & 0.4714          & 2.2960          & 75.4734         & 4.87                                                                 & 160.10                                                       \\
DiSMEC                     & 0.4878          & 1.1519          & 80.3084         & 2.36                                                                 & 164.63                                                       \\
\alg        & 0.4802          & 0.8938          & \textbf{6.0104} & 1.86                                                                 & \textbf{12.52}                                               \\
\alg-t      & \textbf{0.4475} & 0.8741          & 32.2013         & 1.95                                                                 & 71.96                                                       \\ \bottomrule
\end{tabular}
		}
		\end{minipage}%
		\begin{minipage}[t]{0.5\linewidth}
		\resizebox{0.98\linewidth}{!}
		{
\begin{tabular}{@{}lccccc@{}}
\toprule
Method & \begin{tabular}[c]{@{}c@{}}WP-regret-l\\ @k\end{tabular}    & \begin{tabular}[c]{@{}c@{}}2*XMAD-l\\ @2k\end{tabular}      & MAD             & \begin{tabular}[c]{@{}c@{}}2*XMAD-l@2k / \\WP-regret-l@k\end{tabular} & \begin{tabular}[c]{@{}c@{}}MAD /\\ WP-regret-l@k\end{tabular} \\ \midrule
\multicolumn{6}{c}{YahooMovie-8K}                                                                                                                                                                                     \\ \midrule
PfastreXML                 & 0.5665          & 0.8875          & 8.4893          & 1.57                                                                 & 14.99                                                        \\
Parabel                    & 0.5693          & 0.8850          & \textbf{6.3913} & \textbf{1.55}                                                        & \textbf{11.23}                                               \\
LEML                       & 0.4933          & 0.9571          & 36.0738         & 1.94                                                                 & 73.13                                                       \\
1-vs-all-LS                        & 0.4943          & 0.9288          & 47.1887         & 1.88                                                                 & 95.47                                                       \\
RankSVM                    & 0.4738          & 2.0235          & 76.6099         & 4.27                                                                 & 161.69                                                       \\
DiSMEC                     & 0.4760          & 0.9003          & 38.0254         & 1.89                                                                 & 79.89                                                       \\
XLR                        & 0.6013          & 1.0423          & 17.0795         & 1.73                                                                 & 28.40                                                       \\
\alg       & 0.4676          & \textbf{0.8847} & 6.7809          & 1.89                                                                 & 14.50                                                        \\
\alg-t      & \textbf{0.4664} & 0.8964          & 12.1071         & 1.92                                                                 & 25.96                                                       \\ \midrule
\multicolumn{6}{c}{DSA-130K}                                                                                                                                                                                          \\ \midrule
PfastreXML                 & 0.0289          & 0.0500          & 0.3291          & 1.73                                                                 & 11.39                                                        \\
Parabel                    & 0.0266          & 0.1230          & 0.8827          & 4.62                                                                 & 33.18                                                        \\
LEML                       & 0.0361          & \textbf{0.0486} & 0.2828          & \textbf{1.35}                                                        & 7.83                                                        \\
DiSMEC                     & 0.0265	       & 0.1547             &	4.0878 & 5.84&154.26\\
XLR                        & 0.0402          & 0.9157          & 34.429          & 22.78                                                                & 856.44                                                       \\
\alg       & 0.0259          & 0.0519          & \textbf{0.2482} & 2.00                                                                 & \textbf{9.58}                                               \\
\alg-t      & \textbf{0.0256} & 0.0787          & 0.5862          & 3.07                                                                 & 22.90                                                        \\ \bottomrule
\end{tabular}
		}
		\end{minipage}
		\vspace{-2mm}
\end{table*}
\begin{table}[t!]

\centering
\captionsetup{font=small}
		\caption{The ablation study of Parabel leading to per-label \alg-t which clearly outperforms its predecessors on ranking metrics.}
		\label{tab:parabel}
\resizebox{\linewidth}{!}
  {
\begin{tabular}{@{}l|cc|cc|cc@{}}
\toprule
Method                        & \begin{tabular}[c]{@{}c@{}}Rating-l@5 \\ (\%)\end{tabular}   & \begin{tabular}[c]{@{}c@{}}AUC-l@5 \\ (\%)\end{tabular}   & \begin{tabular}[c]{@{}c@{}}CTR-l@5 \\ (\%)\end{tabular} & \begin{tabular}[c]{@{}c@{}}AUC-l@5 \\ (\%)\end{tabular} & \begin{tabular}[c]{@{}c@{}}Rating-l@5 \\ (\%)\end{tabular}   & \begin{tabular}[c]{@{}c@{}}AUC-l@5 \\ (\%)\end{tabular} \\ \midrule
                              & \multicolumn{2}{c}{YahooMovie-8K} & \multicolumn{2}{c}{DSA-130K}   & \multicolumn{2}{c}{MovieLens-138K} \\ \midrule
Parabel-logloss               & 8.99       & 32.00        & 33.58     & 28.46       & 2.36        & 48.20          \\
Pointwise \alg & 9.47       & 34.00         & 34.59     & 27.13    & 3.89        & 52.35      \\
Labelwise \alg & 25.86      & 35.00         & 35.66    & \textbf{28.51}       & 48.94      & 66.99         \\
Labelwise \alg-t & \textbf{26.05}      & \textbf{35.33}         & \textbf{36.32}    & 28.45       & \textbf{49.29}       & \textbf{67.36}          \\\bottomrule
\end{tabular}}
\end{table}
\begin{table}
\vspace{-2mm}
\centering
    \captionsetup{font=small}
		\caption{\alg significantly improves query coverage over the existing ensemble for DSA on Bing. \textit{Note}: Cov: Query Coverage, CY: Click Yield, IY: Impression Yield, BR: Bounce Rate}
	\label{tab:bing}
	\resizebox{\linewidth}{!}
	{
		\begin{tabular}{lccccc}
			\toprule
			Method			& \multicolumn{1}{c}{Relative}	& \multicolumn{1}{c}{Relative}	& \multicolumn{1}{c}{Relative}	& \multicolumn{1}{c}{Relative} & \multicolumn{1}{c}{Relative}	\\
															&  \multicolumn{1}{c}{Cov (\%)}		& \multicolumn{1}{c}{CY (\%)}		& \multicolumn{1}{c}{IY (\%)}   & \multicolumn{1}{c}{CTR (\%)} & \multicolumn{1}{c}{BR (\%)}   \\
			\midrule
			Pointwise \alg				& -		& 105		& 105		& 100 & 100\\
			Labelwise \alg				& \textbf{127}		& \textbf{148}		& \textbf{150}		& 98 &  100	\\
			\bottomrule
		\end{tabular}
	}
\vspace{-3mm}
\end{table}
{\bf Results on benchmark datasets}: Table~\ref{tab:res} compares \alg's performance to diverse baselines on datasets belonging to tagging, recommendation and DSA applications. In terms of prediction accuracy, \alg consistently achieves close to best performance in terms of WP@$5$, Tau@$5$ as well as XMAD@$5$ metrics. In particular, \alg can be up to 2.4\%, 3.89\% and 2x better than all baselines in WP@$5$, Tau@$5$ and XMAD@$5$ respectively. 

On most tagging datasets, \alg scores within 2\% of the state-of-the-art ProXML in terms of the popular PSP@$5$ metric but can be up to 1000x faster during both training and prediction. 

\alg consistently outperforms extreme classifiers like Parabel and DiSMEC which train only on binary labels. In particular, \alg can be up to 9\% and 45\% better than Parabel over pointwise and labelwise datasets in terms of WP@$5$. The larger gains on labelwise datasets are due to pointwise prediction in Parabel which can lead to low label coverage, especially on datasets like MovieLens with only 8K test points but around 140K labels. Owing to similar classifier architectures, \alg can be highly efficient just like Parabel. \alg is at most 3.75x and 4.8x slower during training and prediction and has at most 2.15x the model size as Parabel.

Owing to their high scalability, both Parabel and \alg scale to the largest DSA-1M dataset where none of the other approaches scale. On this dataset, \alg has 50\% smaller XMAD@$5$ than Parabel.

\alg-t denotes the re-ranked \alg where the predicted relevance estimates are combined with tail classifier scores to improve ranking performance over more informative tail labels. \alg-t consistently improves performance over \alg since most XR datasets are dominated by tail labels. \alg-t can be up to 5.66\% and 5.58\% better than \alg in terms of PSP@$5$ and Tau@$5$. However, \alg-t often increases XMAD@$5$ over \alg since tail classifiers are not regressors but are good generative classifiers which and therefore increase regression errors. Since the tail classifiers are extremely efficient to train and the re-ranking step is only applied to a small number (100s) of labels with high relevance estimates from \alg, \alg-t can be very efficient with 1.1, 1.96 and 2.8 times the training time, prediction time and model size as \alg in worst case.

Additional results for WP@k, Tau@k where k=1,3, nDCG@5 and XRMSE@5 are available in the \suppurl.

{\bf Filtering and re-ranking}: The accurate relevance weight estimates that \alg produces can be used for many downstream tasks such as filtering and re-ranking as discussed in Sections~\ref{sec:introduction} and ~\ref{sec:metrics}. Table~\ref{tab:prcurves} reports (1) AUPRC which measures the quality of filtering and (2) WP-rerank@$5$ which measures the quality of reranking with tail classifiers by using the relevance estimates generated by (a) Parabel, (b) \alg and (c) \alg-zero which corrupts \alg's estimates by setting all relevances to almost $0$ while maintaining the same rankings. As can be seen, \alg consistently outperforms Parabel and \alg-zero, both of which have higher regression errors as measured by XMAD@$5$, during both filtering and re-ranking. \alg-zero's results demonstrate that just accurate ranking, as measured by the WP@$5$ column, is not enough for good filtering and re-ranking performance and that low regression errors are also necessary. Furthermore, regression errors measured in terms of traditional MAD are not reliable since MAD is sensitive to the sparsity in relevances and can in fact be lower for corrupted relevances such as in \alg-zero. Figures showing the AUPRC plots can be found in ~\suppurl.

{\bf Analysis of ranking errors and regression metrics}:
Table~\ref{tab:regret} presents the relationship between XMAD \& MAD to the ranking error (WP-regret@k). Table~\ref{tab:regret} shows that, across all the baselines, 2*XMAD@2k is a much better upper bound for WP-regret@k compared to the traditional MAD. Particularly, on regression and classification techniques, 2*XMAD@2k is 1.35-5.84 times the WP-regret@k while MAD can be up to 69x larger than 2*XMAD@2k. In general, ranking baselines (RankSVM, XLR) do not produce good regression values making the ratio of 2*XMAD@2k and MAD to WP-regret@k much higher. Lastly, for the dense score prediction algorithms like 1-vs-all-LS, MAD is significantly high since it sums up the errors across all labels.

{\bf Ablation Studies}: To test the effectiveness of the proposed \alg along with its novel labelwise prediction algorithm, experiments were done to show the boost due to each of the factors. First, the extension of Parabel-logloss to utilize label weights lead to pointwise \alg which improved the ranking metrics up to 1.5\% over Parabel across the 3 labelwise datasets showing that \alg can learn better from relevance weights. Further, when \alg was coupled with the novel labelwise prediction algorithm, the gains were up to 16\%, 1.1\% and 45\% on YahooMovie-8K, DSA-130K, and MovieLens-138K respectively due to higher label coverage. Lastly, the use of tail classifiers with \alg (\alg-t) further increased the ranking performance by up to 0.7\% over labelwise \alg.

{\bf DSA Results}: Table~\ref{tab:res} shows the offline evaluation on DSA-130K and DSA-1M while Table~\ref{tab:bing} showcases the results of the live deployment of labelwise \alg in Bing DSA pipeline. Even though few of the extreme classification techniques could scale to DSA-130K, the live deployment requires the techniques to scale to tens of millions of labels (queries) and data points (ads). In the actual deployment only PfastreXML, Parabel, and \alg were able to scale.

Table~\ref{tab:bing} compares \alg's performance to the existing DSA ensemble, consisting of BM25 information retrieval based algorithm~\citep{Jones00} and PfastreXML when deployed on Bing. Both pointwise and labelwise \alg were deployed and evaluated. Pointwise \alg increased RPM, CY, and IY by $5\%$ while maintaining the CTR and BR. Finally, the labelwise \alg improves query coverage by $27\%$ along with a $48\%$ and $50\%$ increase in click yield and impression yields at a cost of only $2\%$ reduction in CTR.

\vspace{-1mm}
\section{Conclusions}

This paper proposed a new learning paradigm called eXtreme Regression (XR) which provides a scalable solution to many real-world recommendation and ranking problems such as tagging, recommendation, DSA {\it etc.} XR involves learning to accurately predict the numerical relevance weights of an extremely large number of labels with respect to a data point. These weights not only induce an accurate ranking but are also useful for subsequent filtering and re-ranking steps. To effectively solve XR problems, this paper also develops a new evaluation metric called XMAD@$k$ and a new algorithm called XReg. \alg consistently outperforms the state-of-the-art extreme classifiers as well as large-scale regressors and rankers in terms of ranking accuracies and efficiently scales to datasets with millions of data points and labels. Deployment of \alg on DSA in Bing resulted in a relative gain of 27\% in query coverage.

\vspace{-1mm}
\section{Acknowledgements}We are grateful to Kunal Dahiya, Prateek Jain, Nagarajan Natarajan, Deepak Saini and Harsha Vardhan Simhadri for helpful discussions, feedback and computing resources.
\vspace{-1mm}

\begin{small}{
\bibliography{local}}
\end{small}
\newpage
\appendix
\newcommand{\T}{\mathcal{T}}
\algrenewcommand{\algorithmiccomment}[1]{\hskip3em{\scriptsize \# #1}}
\algdef{SE}[DOWHILE]{Do}{doWhile}{\algorithmicdo}[1]{\algorithmicwhile\ #1}%

\begin{algorithm}[t!]
\label{alg:spge}
\caption{\alg Labelwise Prediction}
    \begin{flushleft}
		\textbf{Input:}\\\qquad Test data points $\{\v x_i\}_{i=1}^M$\\\qquad Trained tree $\T$\\\qquad Required no. of most relevant test points per label $N$\\\qquad Fraction of test points relevant for each node $\{frac(n)\}_{n=1}^{|nodes(\T)|}$\\\qquad  Multiplicative factor $F$\myneg{10}\algorithmiccomment{$F*frac(n)$ most relevant test points are passed down to node $n$}\\
		\textbf{Output:}\\\qquad Predicted test points for each label $\{pred(l)\}_{l=1}^L$\\
		\textbf{Initialize:}\\\qquad$points(1)\leftarrow\{1,\cdots,M\}$\myneg{10}\algorithmiccomment{$points(n)$ is set of test points passed to node $n$}\\\qquad$\hat{z}_{i1}=1.0 ~~~\forall i \in \{1,\cdots,M\}$\myneg{10}\algorithmiccomment{All test points visit the root node with probability $1$}\\
	\end{flushleft}
	\begin{algorithmic}
    	\For{$n \in \{1,\cdots,|nodes(\T)|\}$} \algorithmiccomment{Breadth-first exploration of the tree}
    	    \If{$n \in internalnodes(\T)$}
    	        \For{$n' \in children(n)$} \algorithmiccomment{Iterate over children nodes of $n$}
    	            \For{$i \in points(n)$}
    	                \State $\hat{z}_{in'} \leftarrow \hat{z}_{in}*Sigmoid(\v w_{n'}^{\top}\v x_i)$ \myneg{18}\algorithmiccomment{$Sigmoid(x)=\frac{1}{1+\exp(-x)}$}
    	            \EndFor
    	            \State $points(n')\leftarrow$ \Call{retain top}{$\{\hat{z}_{in'}\}_{i \in points(n)},F*frac(n')$}
    	        \EndFor
    		\ElsIf{$n \in leafnodes(\T)$}
    		    \For{$l \in labels(n)$} \algorithmiccomment{Iterate over labels in leaf node $n$}
    		        \For{$i \in points(n)$}
    	                \State $\hat{y}_{il} \leftarrow \hat{z}_{in}*Sigmoid(\v w_l{\top}\v x_i)$
    	            \EndFor
    	            \State $pred(l)\leftarrow$ \Call{retain top}{$\{\hat{y}_{il}\}_{i \in points(n)},N$}
                \EndFor
            \EndIf
    	\EndFor
        \State Note: In case there are multiple trees in ensemble, the probability predictions estimated by all trees are averaged for a each test point, label pair and top $N$ test points are outputted for each label\\
        \Return $\{pred(l)\}_{l=1}^L$
        \\
        \Procedure{retain top}{$\{s_i\}_{i=1}^I,N$}
			\State $R \leftarrow \mbox{argsort}(\{s_i\}_{i=1}^I, \mbox{comparator} \leftarrow s_{i_1}>s_{i_2})$\algorithmiccomment{Sort the test points in decreasing order of node or label probabilities}
			\State $\mathcal{B} \leftarrow  \{R[1],..,R[N]\}$ \\
			\Return $\mathcal{B}$
		\EndProcedure
  \end{algorithmic}
\end{algorithm}


\begin{table*}[t!]
\centering
\captionsetup{font=small}
	\caption{\alg has the best or close to the best ranking and regression performance across all the datasets compared to state-of-the-art extreme classifiers and large-scale regressors and rankers. Re-ranking with tail classifiers (\alg-t) further improves the accuracies. PSP@$k$, CTR@$k$ and Rating@$k$ are variants of WP@$k$ as discussed in Section~\ref{sec:metrics}. "-p": pointwise, "-l": labelwise.}
	\label{tab:auxres}
	\begin{minipage}[t]{0.5\linewidth}
		\resizebox{0.95\linewidth}{!}
		{
		\begin{tabular}{@{}lcccccc@{}}
\toprule
Method     & \begin{tabular}[c]{@{}c@{}}PSP-p@1 \\ (\%)\end{tabular} & \begin{tabular}[c]{@{}c@{}}PSP-p@3 \\ (\%)\end{tabular} & \begin{tabular}[c]{@{}c@{}}Tau-p@1 \\ (\%)\end{tabular} & \begin{tabular}[c]{@{}c@{}}Tau-p@3 \\ (\%)\end{tabular} & \begin{tabular}[c]{@{}c@{}}nDCG-p@5 \\ (\%)\end{tabular} & XRMSE-p@5        \\ \midrule
\multicolumn{6}{c}{\textbf{BibTex}}                                                                                                                                                                                                                                  \\ \midrule
PfastreXML & 52.43                                                   & 53.41                                                   & 41.31                                                   & 48.36                                                 &56.41  & 0.3813          \\
Parabel    & 50.88                                                   & 52.42                                                   & 36.57                                                   & 46.28                                              &54.58     & 0.4104          \\
LEML       & 51.30                                                   & 52.17                                                   & 39.32                                                   & 47.31                                                  &54.10 & 0.3935          \\
1-vs-all-LS        & \textbf{53.50}                                          & \textbf{55.10}                                          & \textbf{39.91}                                          & \textbf{49.31}                                   & \textbf{57.30}       & 0.3834          \\
RankSVM    & 49.31                                                   & 51.79                                                   & 39.22                                                   & 47.07                                                 &54.97  & 0.7228          \\
DiSMEC     & 50.88                                                   & 52.52                                                   & 36.66                                                   & 46.34                                                 &54.54  & 0.4104          \\
ProXML     & 50.10                                                   & 52.00                                                   & -                                                       & -                                                    &-   & -               \\
\alg        & 49.66                                                   & 52.66                                                   & 38.48                                                   & 47.09                                                 & 54.98 & 0.3958          \\
\alg-t      & 49.86                                                   & 53.04                                                   & 38.68                                                   & 47.33                                                 &55.14  & \textbf{0.3805} \\ \midrule
\multicolumn{6}{c}{\textbf{EURLex-4k}}                                                                                                                                                                                                                               \\ \midrule
PfastreXML & 40.16                                                   & 43.07                                                   & 46.97                                                   & 46.50                                                 & 43.64 & 0.2188          \\
Parabel    & 36.36                                                   & 44.04                                                   & 40.96                                                   & 46.03                                                  &44.78 & 0.4673          \\
LEML       & 27.20                                                   & 30.15                                                   & 30.47                                                   & 33.62                                                  &30.73 & 0.2406          \\
1-vs-all-LS        & \textbf{47.02}                                          & \textbf{50.85}                                          & \textbf{54.12}                                          & \textbf{52.45}                     &   \textbf{50.82}                  & \textbf{0.2006} \\
RankSVM    & 39.52                                                   & 43.82                                                   & 51.23                                                   & 49.79                                                &  44.49 & 1.2093          \\
DiSMEC     & 37.58                                                   & 45.92                                                   & 42.32                                                   & 47.56                                                 & 46.73 & 0.4771          \\
ProXML     & 45.20                                                   & 48.50                                                   & -                                                       & -                                                       & -               \\
\alg        & 44.00                                                   & 47.44                                                   & 51.69                                                   & 50.51                                               & 47.99   & 0.2127          \\
\alg-t      & 45.23                                                   & 48.51                                                   & 53.06                                                   & 51.59                                                 &48.89  & 0.2338          \\ \midrule
\multicolumn{6}{c}{\textbf{Wiki10-31K}}                                                                                                                                                                                                                              \\ \midrule
PfastreXML & 12.94                                                   & 14.80                                                   & 11.93                                                   & 17.53                                               &  15.13  & 0.5925          \\
Parabel    & 11.66                                                   & 12.73                                                   & 13.36                                                   & 17.18                                                 & 13.13  & 0.7201          \\
LEML       & 11.25                                                   & 12.38                                                   & 15.29                                                   & 18.40                                               & 12.58    & 0.5938          \\
1-vs-all-LS        & \textbf{26.78}                                          & 23.06                                                   & 36.59                                                   & 29.50                                        &23.01           & \textbf{0.5691} \\
RankSVM    & 21.06                                                   & 18.99                                                   & 32.63                                                   & 27.58                                            &19.05       & 1.2279          \\
DiSMEC     & 11.91                                                   & 14.09                                                   & 14.41                                                  &  19.47                                             &14.63     & 0.7140          \\
\alg        & 17.33                                                   & 16.73                                                   & 26.76                                                   & 25.01                                               &16.98    & 0.5931          \\
\alg-t      & 25.92                                                   & \textbf{23.56}                                          & \textbf{38.72}                                          & \textbf{33.09}                                      &\textbf{23.40}    & 0.5722          \\ \midrule
\multicolumn{6}{c}{\textbf{WikiLSHTC-325K}}                                                                                                                                                                                                                          \\ \midrule
PfastreXML & 25.67                                                   & 26.57                                                   & 31.11                                                   & 34.15                                            &27.09       & 0.1922          \\
Parabel    & 26.71                                                   & 33.16                                                   & 28.05                                                   & 36.99                                               &33.48    & 0.3130          \\
DiSMEC     & 29.10                                                   & 35.60                                                   & -                                                       & -                                                    &-   & -               \\
ProXML     & 34.80                                                   & 37.70                                                   & -                                                       & -                                          &-             & -               \\
\alg        & 32.36                                                   & 34.36                                                   & 36.59                                                   & 39.10                                       & 35.13            & \textbf{0.1877} \\
\alg-t      & \textbf{36.85}                                          & \textbf{37.98}                                          & \textbf{41.41}                                          & \textbf{41.96}                                  &\textbf{38.87}        & 0.3241          \\ \midrule
\multicolumn{6}{c}{\textbf{Amazon-670K}}                                                                                                                                                                                                                             \\ \midrule
PfastreXML & 24.52                                                   & 26.65                                                   & 28.18                                                   & 29.30                                               &27.36    & 0.4356          \\
Parabel    & 25.43                                                   & 29.45                                                   & 20.54                                                   & 26.56                                              &30.72     & 0.4640          \\
DiSMEC     & 25.82                                                   & 30.20                                                   & 20.40                                                   & 26.77                                            &31.89       & 0.4582          \\
ProXML     & 30.80                                                   & 32.80                                                   & -                                                       & -                                                    &-   & -               \\
\alg        & 29.12                                                   & 31.19                                                   & 31.69                                                   & 32.63                                        &32.01           & \textbf{0.4189} \\
\alg-t      & \textbf{31.16}                                          & \textbf{32.71}                                          & \textbf{33.83}                                          & \textbf{34.28}                                   &\textbf{33.34}       & 0.4639          \\ \bottomrule
\end{tabular}
		}
		\end{minipage}%
		\begin{minipage}[t]{0.5\linewidth}
		\resizebox{0.95\linewidth}{!}
		{
		\begin{tabular}{@{}lcccccc@{}}
\toprule
Method     & \begin{tabular}[c]{@{}c@{}}CTR-p@1 \\ (\%)\end{tabular} & \begin{tabular}[c]{@{}c@{}}CTR-p@3 \\ (\%)\end{tabular} & \begin{tabular}[c]{@{}c@{}}Tau-p@1 \\ (\%)\end{tabular} & \begin{tabular}[c]{@{}c@{}}Tau-p@3 \\ (\%)\end{tabular} & \begin{tabular}[c]{@{}c@{}}nDCG-p@5 \\ (\%)\end{tabular} & XRMSE-p@5       \\ \midrule
\multicolumn{6}{c}{\textbf{SSA-130K}}                                                                                                                                                                                                                                \\ \midrule
PfastreXML & 21.34                                                   & 25.24                                                   & 22.33                                                   & 23.1                                                &25.56    & \textbf{0.0817} \\
Parabel    & 21.95                                                   & 28.87                                                   & 26.98                                                   & \textbf{28.83}                                      &29.22    & 0.1636          \\
LEML       & 3.79                                                    & 5.11                                                    & 7.2                                                     & 7.61                                               &5.50     & 0.0835          \\
RankSVM & 8.92 & 10.96 & 13.14 & 13.74 & 11.51 & 2.7945 \\
DiSMEC & 21.36 & 28.41 & 25.68 & 27.64 & 28.85 & 0.1746 \\
\alg        & 24.69                                                   & 29.02                                                   & 27.52                                                   & 27.71                                            &29.64       & 0.0826          \\
\alg-t      & \textbf{24.7}                                           & \textbf{29.22}                                          & \textbf{27.32}                                          & 27.83                                             &\textbf{29.93}      & 0.1225          \\ \midrule\midrule
Method     & \begin{tabular}[c]{@{}c@{}}Rating-l@1 \\ (\%)\end{tabular}   & \begin{tabular}[c]{@{}c@{}}Rating-l@3 \\ (\%)\end{tabular}   & \begin{tabular}[c]{@{}c@{}}Tau-l@1 \\ (\%)\end{tabular} & \begin{tabular}[c]{@{}c@{}}Tau-l@3 \\ (\%)\end{tabular} & \begin{tabular}[c]{@{}c@{}}nDCG-l@5 \\ (\%)\end{tabular} & XRMSE-l@5       \\ \midrule
\multicolumn{6}{c}{\textbf{YahooMovie-8K}}                                                                                                                                                                                                                           \\
PfastreXML & 11.5                                                    & 9.9                                                     & 22.29                                                   & 20.21                                             &10.36      & 0.7047          \\
Parabel    & 11.28                                                   & 9.73                                                    & 30.11                                                   & 29.08                                              &10.03     & 0.7054          \\
LEML       & 21.33                                                   & 21.06                                                   & 28.81                                                   & 29.5                                                &21.55    & 0.6851          \\
1-vs-all-LS        & 22.75                                                   & 21.02                                                   & 34.9                                                    & 32.59                                &21.76                   & 0.6791          \\
RankSVM    & 24.89                                                   & 23.16                                                   & 36.99                                                   & 34.7                                               &24.53     & 1.0613          \\
DiSMEC     & 23.76                                                   & 23.19                                                   & 34.62                                                   & 33.45                                              &24.10     & 0.6826          \\
XLR &        3.87 & 4.2 & 12.44 & 11.78 & 4.40&0.7182                 \\
\alg        & 26.49                                                   & 24.77                                                   & 39.02                                                   & 35.8                                                 &25.76   & 0.6944          \\
\alg-t      & \textbf{26.53}                                          & \textbf{24.86}                                          & \textbf{39.28}                                          & \textbf{36.42}                                     &\textbf{25.90}     & \textbf{0.6772} \\ \midrule
\multicolumn{6}{c}{\textbf{MovieLens-138K}}                                                                                                                                                                                                                          \\ \midrule
PfastreXML & 9.03                                                    & 7.82                                                    & 25.92                                                   & 23.88                                               &7.63    & 0.9253          \\
Parabel    & 5.95                                                    & 4.25                                                    & 40.67                                                   & 39.08                                               &4.03    & 0.9254          \\
LEML       & 46.51                                                   & 44.89                                                   & 69.58                                                   & 66.96                                               &43.97    & 0.8773          \\
1-vs-all-LS        & 46.94                                                   & 43.88                                            &43.17       & 69.28                                                   & 65.92                                                   & 0.8882          \\
DiSMEC     & 50.85                                                   & 47.05                                                   & 65.45                                                   & 62.93                                             &46.49      & 0.8909          \\
XLR &        14.49 & 10.31 & 31.61 & 21.5 & 10.55& 0.9184                 \\
\alg        & 54.65                                                   & 50.83                                                   & 71.59                                                   & 68.83                                               & 50.16    & 0.8793          \\
\alg-t      & \textbf{55.04}                                          & \textbf{51.21}                                          & \textbf{72.07}                                          & \textbf{69.3}                                   & \textbf{50.52}        & \textbf{0.8337} \\ \midrule\midrule
Method     & \begin{tabular}[c]{@{}c@{}}CTR-l@1 \\ (\%)\end{tabular} & \begin{tabular}[c]{@{}c@{}}CTR-l@3 \\ (\%)\end{tabular} & \begin{tabular}[c]{@{}c@{}}Tau-l@1 \\ (\%)\end{tabular} & \begin{tabular}[c]{@{}c@{}}Tau-l@3 \\ (\%)\end{tabular} & \begin{tabular}[c]{@{}c@{}}nDCG-l@5 \\ (\%)\end{tabular} & XRMSE-l@5       \\ \midrule
\multicolumn{6}{c}{\textbf{DSA-130K}}                                                                                                                                                                                                                                \\ \midrule
PfastreXML & 18.15                                                   & 23.7                                                    & \textbf{26.77}                                          & \textbf{30.99}                                     &23.93     & \textbf{0.0647} \\
Parabel    & 19.97                                                   & 28.06                                                   & 23.44                                                   & 26.04                                         &28.13          & 0.1091          \\
LEML       & 3.94                                                    & 6.9                                                     & 5.35                                                    & 6.46                                            &7.54        & 0.0657          \\
XLR &        0.03 & 0.07 & 0.1 & 0.1 & 0.07& 0.4837                 \\
DiSMEC & 18.94 & 27.52 & 22.20 & 25.25 & 27.70 & 0.1201 \\ 
\alg        & 22.07                                                   & 29.73                                                   & 23.73                                                   & 26.08                                            &29.95       & 0.0654          \\
\alg-t      & \textbf{22.41}                                          & \textbf{30.1}                                           & 23.98                                                   & 26.13                                          &\textbf{30.43}         & 0.0744          \\ \midrule
\multicolumn{7}{c}{\textbf{DSA-1M}} \\ \midrule
Parabel & 25.78 & 33.15 & 27.93 & 29.38 & 33.28& 0.1218\\ 
\alg & 26.75 & 33.06 & 28.67 & 29.51 & 33.36 & \textbf{0.0806}\\
\alg-t & \textbf{27.83} & \textbf{34.27} & \textbf{29.68} &\textbf{30.18} &\textbf{34.55}& 0.0892\\\bottomrule
\end{tabular}
		}
		\end{minipage}
\end{table*}

\begin{table*}[t!]
\centering
\captionsetup{font=small}
	\caption{Hyperparameter tuning for \# trees ($T$), Max leaf labels ($M$), Beam width ($P$) and points reaching leaf node per label in labelwise prediction of \alg. Note: The hyperparameters in bold face are finally chosen for the default setting.}
	\label{tab:hyper}
	\begin{minipage}[t]{0.5\linewidth}
		\resizebox{0.95\linewidth}{!}
		{
		\begin{tabular}{@{}ccccccc@{}}
\toprule
\begin{tabular}[c]{@{}c@{}}\# \\ trees\end{tabular} & \begin{tabular}[c]{@{}c@{}}WP@5\\ (\%)\end{tabular} & \begin{tabular}[c]{@{}c@{}}Tau@5\\ (\%)\end{tabular} & XMAD@5 & \begin{tabular}[c]{@{}c@{}}Training\\  time (hrs)\end{tabular} & \begin{tabular}[c]{@{}c@{}}Test time\\ /point (ms)\end{tabular} & \begin{tabular}[c]{@{}c@{}}Model\\ Size (GB)\end{tabular} \\ \midrule
\multicolumn{7}{c}{\textbf{EURLex-4K (pointwise)}}                                                                                                                                                                                                                                                                                                                       \\ \midrule
1                                                   & 48.05                                               & 51.14                                                & 0.1899 & 0.0307                                                         & 0.3911                                                          & 0.0125                                                    \\
\textbf{3}                                                   & 49.72                                               & 52.86                                                & 0.1849 & 0.0642                                                         & 1.2899                                                          & 0.0378                                                    \\
5                                                   & 50.25                                               & 53.24                                                & 0.1836 & 0.0995                                                         & 2.638                                                           & 0.0629                                                    \\ 
7 & 50.44 & 53.42 & 0.1831 & 0.1543 & 3.4703 & 0.0881\\\midrule
\multicolumn{7}{c}{\textbf{Amazon-670K (pointwise)}}                                                                                                                                                                                                                                                                                                                     \\ \midrule
1                                                   & 30.50                                               & 32.31                                                & 0.3956 & 0.6788                                                         & 0.6551                                                          & 1.1478                                                    \\
\textbf{3}                                                   & 33.24                                               & 34.72                                                & 0.3869 & 1.4925                                                         & 2.4633                                                          & 3.4186                                                    \\
5                                                   & 34.00                                               & 35.45                                                & 0.3847 & 2.1499                                                         & 6.9283                                                          & 5.6978                                                    \\ 
7 & 34.37 & 35.86 & 0.3837 & 3.4298 & 8.564 & 7.9768\\\midrule
\multicolumn{7}{c}{\textbf{DSA-130K (labelwise)}}                                                                                                                                                                                                                                                                                                                        \\ \midrule
1                                                   & 33.64                                               & 27.52                                                & 0.0448 & 0.2165                                                         & 1.8552                                                          & 0.2624                                                    \\
\textbf{3}                                                   & 35.66                                               & 28.51                                                & 0.0439 & 0.4570                                                         & 7.4715                                                          & 0.7871                                                    \\
5                                                   & 36.24                                               & 28.94                                                & 0.0436 & 0.6836                                                         & 16.3785                                                         & 1.3117                                                    \\ 
7 & 36.55 & 29.2 & 0.0435 & 1.0897 & 21.7585 & 1.8358\\\midrule \midrule
\begin{tabular}[c]{@{}c@{}}Beam\\ width\end{tabular} & \begin{tabular}[c]{@{}c@{}}WP@5\\ (\%)\end{tabular} & \begin{tabular}[c]{@{}c@{}}Tau@5\\ (\%)\end{tabular} & XMAD@5 & \begin{tabular}[c]{@{}c@{}}Training\\  time (hrs)\end{tabular} & \begin{tabular}[c]{@{}c@{}}Test time\\ /point (ms)\end{tabular} & \begin{tabular}[c]{@{}c@{}}Model\\ Size (GB)\end{tabular} \\ \midrule
\multicolumn{7}{c}{\textbf{EURLex-4K (pointwise)}}                                                                                                                                                                                                                                                                                                                        \\ \midrule
5                                                    & 49.6                                                & 52.76                                                & 0.1858 & 0.0627                                                         & 0.6869                                                          & 0.0378                                                    \\
\textbf{10}                                                   & 49.72                                               & 52.86                                                & 0.1849 & 0.0642                                                         & 1.2899                                                          & 0.0378                                                    \\
20                                                   & 49.7                                                & 52.86                                                & 0.1847 & 0.0638                                                         & 2.4982                                                          & 0.0378                                                    \\ 
30 & 49.71 & 52.86 & 0.1847 & 0.0682 & 3.8542 & 0.0378\\\midrule
\multicolumn{7}{c}{\textbf{Amazon-670K (pointwise)}}                                                                                                                                                                                                                                                                                                                        \\ \midrule
5                                                    & 32.77                                               & 34.37                                                & 38.77  & 1.3654                                                         & 1.4467                                                          & 3.4186                                                    \\
\textbf{10}                                                 & 33.24                                               & 34.72                                                & 0.3869 & 1.4925                                                         & 2.4633                                                          & 3.4186                                                    \\
20                                                   & 33.38                                               & 34.82                                                & 0.3866 & 1.3721                                                         & 4.8842                                                          & 3.4186                                                    \\ 
30 & 33.4 & 34.83 & 0.3866 & 1.5508 & 9.1728 & 3.4186\\\bottomrule
\end{tabular}
		}
		\end{minipage}%
		\begin{minipage}[t]{0.5\linewidth}
		\resizebox{0.95\linewidth}{!}
		{
		\begin{tabular}{@{}ccccccc@{}}
\toprule
\begin{tabular}[c]{@{}c@{}}Max leaf\\ labels\end{tabular} & \begin{tabular}[c]{@{}c@{}}WP@5\\ (\%)\end{tabular} & \begin{tabular}[c]{@{}c@{}}Tau@5\\ (\%)\end{tabular} & XMAD@5  & \begin{tabular}[c]{@{}c@{}}Training\\  time (hrs)\end{tabular} & \begin{tabular}[c]{@{}c@{}}Test time\\ /point (ms)\end{tabular} & \begin{tabular}[c]{@{}c@{}}Model\\ Size (GB)\end{tabular} \\ \midrule
\multicolumn{7}{c}{\textbf{EURLex-4K (pointwise)}}                                                                                                                                                                                                                                                                                                                              \\ \midrule
20                                                        & 49.34                                               & 52.75                                                & 0.1845  & 0.0323                                                         & 0.4694                                                          & 0.0494                                                    \\

50 & 49.6 & 52.71 & 0.1859 & 0.0458 & 0.8198 & 0.0428\\
\textbf{100}                                                       & 49.72                                               & 52.86                                                & 0.1849  & 0.0642                                                         & 1.2899                                                          & 0.0378                                                    \\
200                                                       & 50.11                                               & 53.05                                                & 0.1846  & 0.1031                                                         & 2.6368                                                          & 0.0337                                                    \\ \midrule
\multicolumn{7}{c}{\textbf{Amazon-670K (pointwise)}}                                                                                                                                                                                                                                                                                                                            \\ \midrule
20                                                        & 32.26                                               & 33.96                                                & 0.3869  & 0.7514                                                         & 0.7051                                                          & 6.0288                                                    \\
50 & 32.89 & 34.46 & 0.3868 & 1.1171 & 1.7382 & 4.124\\
\textbf{100}                                                       & 33.24                                               & 34.72                                                & 0.3869  & 1.4925                                                         & 2.4633                                                          & 3.4186                                                    \\
200                                                       & 33.56                                               & 34.99                                                & 0.3866  & 2.1426                                                         & 4.3993                                                          & 2.9268                                                    \\ \midrule
\multicolumn{7}{c}{\textbf{DSA-130K (labelwise)}}                                                                                                                                                                                                                                                                                                                               \\ \midrule
20                                                        & 34.76                                               & 28.2                                                 & 0.0448  & 0.2678                                                         & 2.279                                                           & 0.9867                                                    \\
50 & 35.21 & 28.42 & 0.0443 & 0.3516 & 4.8664 & 0.8764\\
\textbf{100}                                                       & 35.66                                               & 28.51                                                & 0.0439  & 0.4570                                                         & 7.4715                                                          & 0.7871                                                    \\
200                                                       & 35.96                                               & 28.63                                                & 0.04351 & 0.6925                                                         & 8.9019                                                          & 0.7128                                                    \\ \midrule \midrule
\begin{tabular}[c]{@{}c@{}}\# per-label\\ points\end{tabular} & \begin{tabular}[c]{@{}c@{}}WP@5\\ (\%)\end{tabular} & \begin{tabular}[c]{@{}c@{}}Tau@5\\ (\%)\end{tabular} & XMAD@5 & \begin{tabular}[c]{@{}c@{}}Training\\  time (hrs)\end{tabular} & \begin{tabular}[c]{@{}c@{}}Test time\\ /point (ms)\end{tabular} & \begin{tabular}[c]{@{}c@{}}Model\\ Size (GB)\end{tabular} \\ \midrule
\multicolumn{7}{c}{\textbf{DSA-130K (labelwise)}}                                                                                                                                                                                                                                                                                                                        \\ \midrule
5 & 35.56 & 28.41 & 0.0438 & 0.4989 & 4.3407 & 0.7871 \\
\textbf{10} & 35.66 & 28.51 & 0.0439 & 0.4570 & 7.4715 & 0.7871 \\
20 & 35.68 & 28.54 & 0.0438 & 0.5164 & 11.1622 & 0.7871 \\
30 & 35.68 & 28.54 & 0.0438 & 0.5135 & 11.5373 & 0.7871 \\\bottomrule

\end{tabular}
		}
		\end{minipage}
\end{table*}
\begin{figure*}[t!]
\centering\hspace*{-2ex}
\begin{tabular}{ccc}
  \includegraphics[width=.44\textwidth, height=0.22\textwidth]{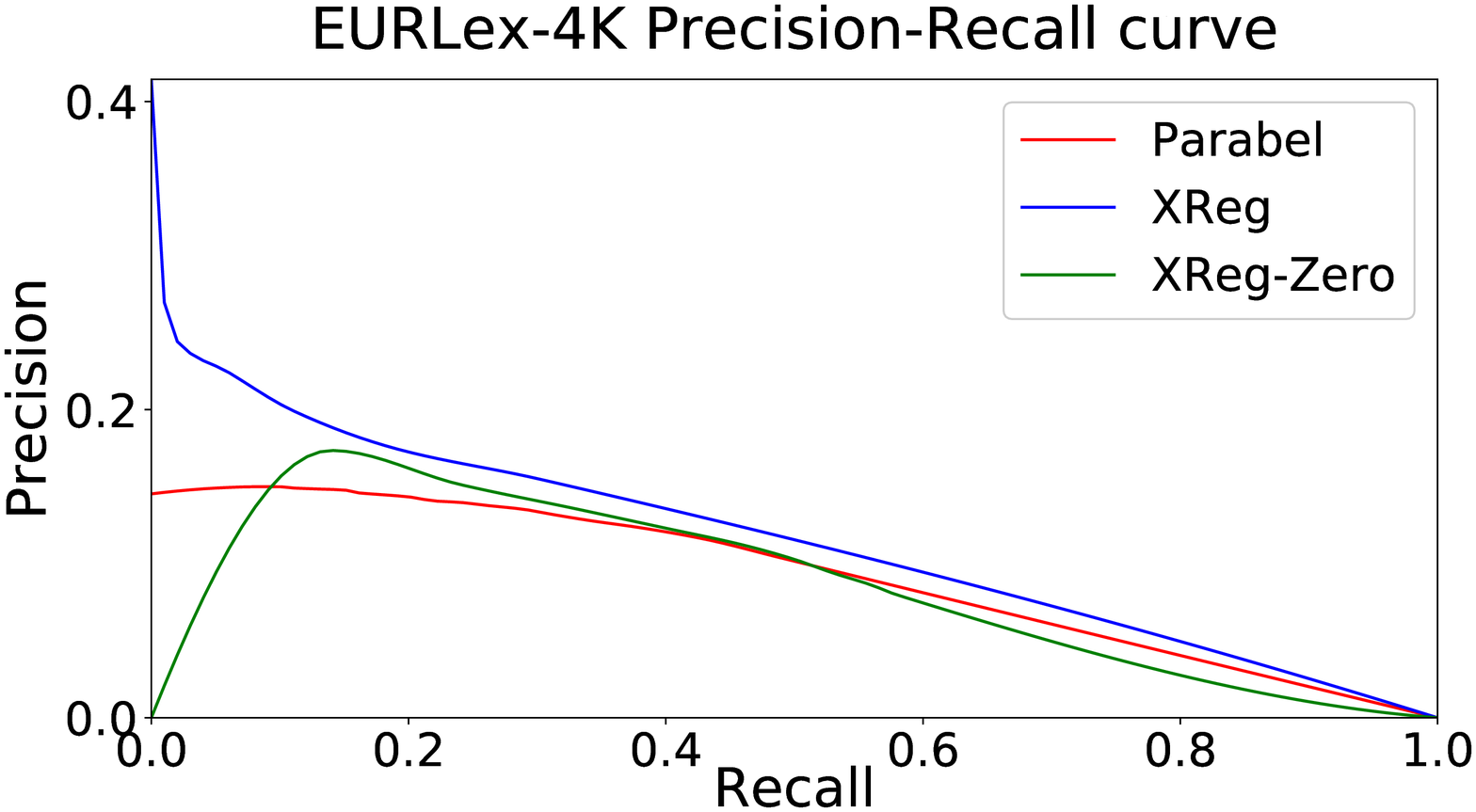}&\hspace{-2ex}
   \includegraphics[width=.44\textwidth, height=0.22\textwidth]{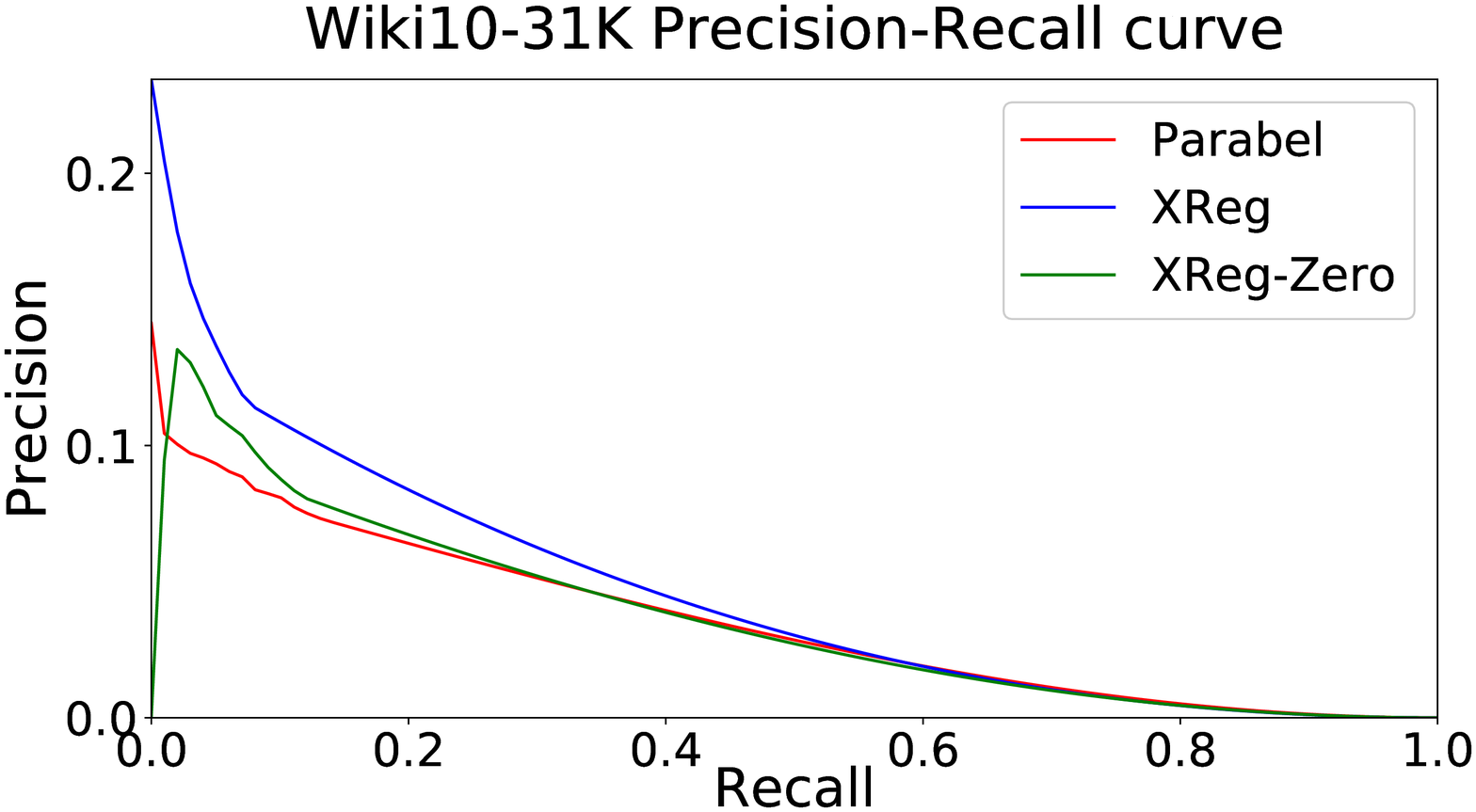}&\hspace{-2ex}\\
  \includegraphics[width=.44\textwidth, height=0.22\textwidth]{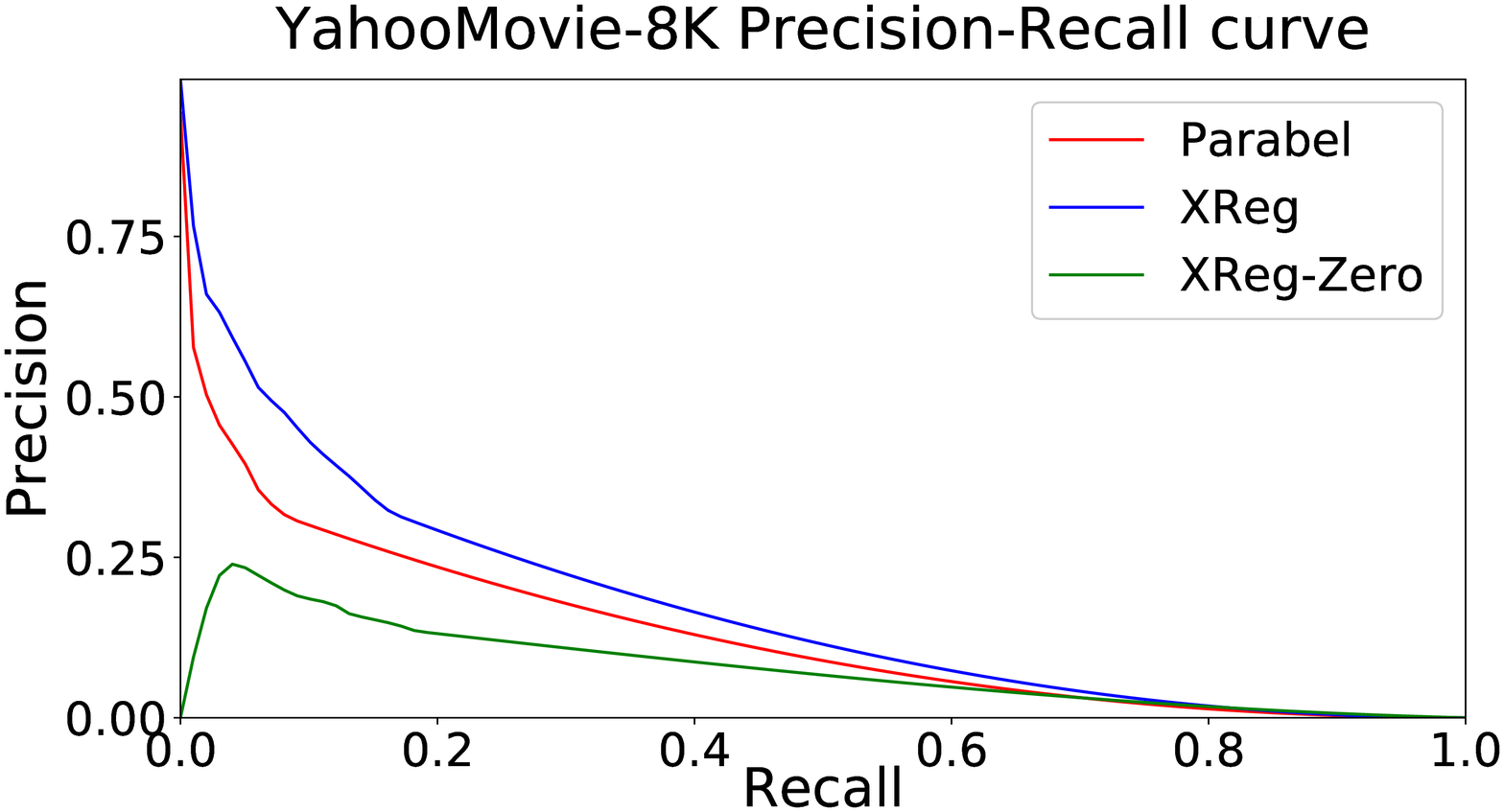}&\hspace{-2ex}
   \includegraphics[width=.44\textwidth, height=0.22\textwidth]{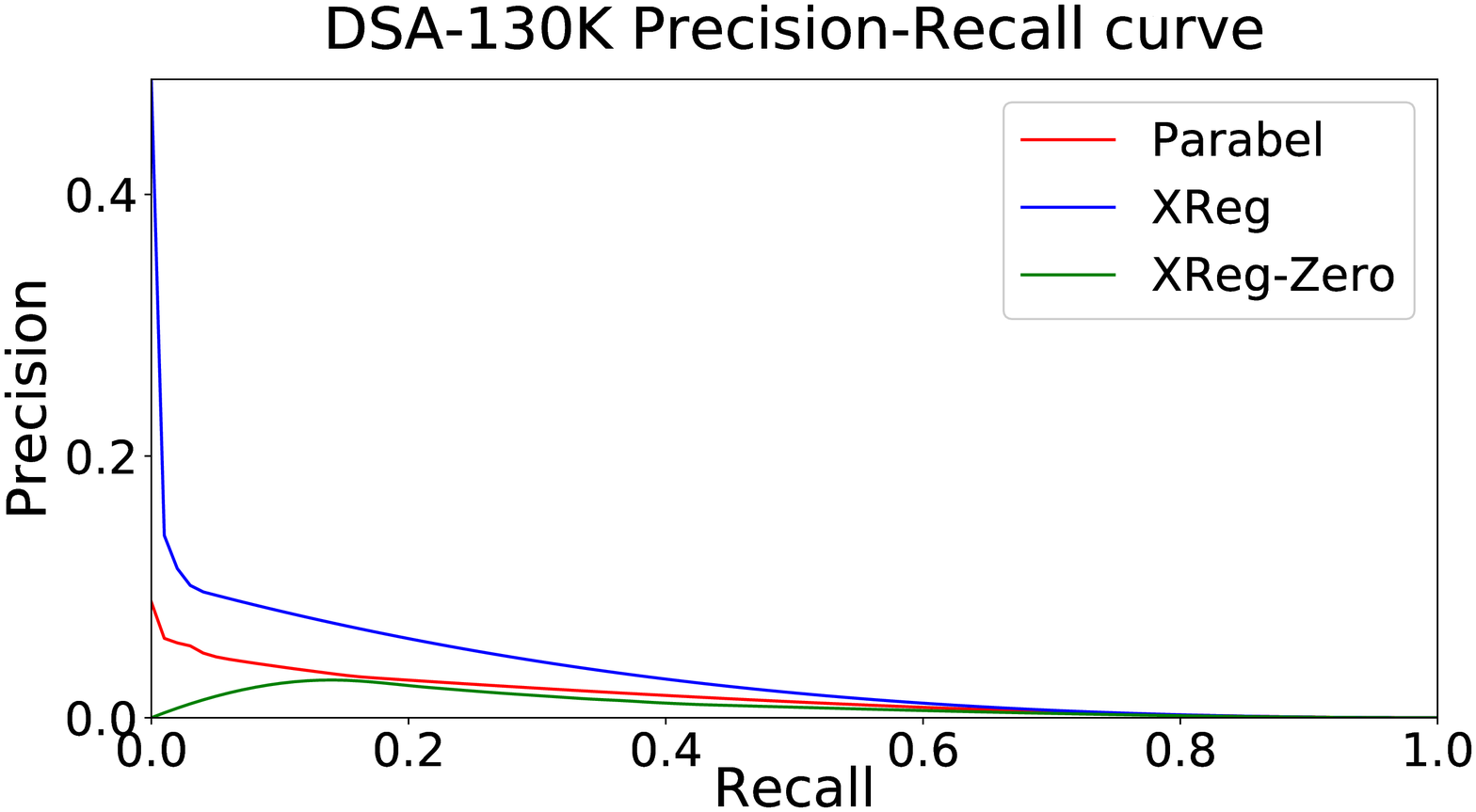}&\hspace{-2ex}\\
\end{tabular}
\captionsetup{font=small}
\vspace{-1mm}
\caption{Precision-Recall curves showing that \alg is consistently better than \alg-Zero and Parabel approaches for precision recall tradeoff.}
  \label{fig:prcurves}
\end{figure*}

\section{Theorems and Proofs}

\begin{lemma}
Given any true and predicted relevance weight vectors $\v y,\hat{\v y} \in [0,\infty)^L$, the following inequality hold true:
\begin{align}
    \label{lem:main}
    &0 \le \frac{M}{2} \le \text{XMAD@}2k(\hat{\v y}_i,\v y_i) \le \text{XRMSE@}2k(\hat{\v y}_i,\v y_i)\\
    &\text{with,}~M = max(\text{Ranking-error@}k,\text{Regression-error@}k)
\end{align}
\end{lemma}
\begin{proof}
The ranking and regression errors are defined as follows
\begin{align}
    \label{eqn:pregretbound}
    &\text{Ranking-error@}k(\hat{\v y}_i,\v y_i) = \frac{1}{k}\sum_{l \in S(\v y_i,k)} y_{il} - \frac{1}{k}\sum_{l \in S(\hat{\v y}_i,k)} y_{il} \\
    &\text{Regression-error@}k(\hat{\v y}_i,\v y_i) = \frac{1}{k}\sum_{l \in S(\hat{\v y}_i,k)} |\hat{y}_{il}-y_{il}|
\end{align}
Since $S(\v y_i,k)$ picks the $k$ largest values of $y_{il}$, Ranking-error@$k(\hat{\v y}_i,\v y_i)\ge 0$ always. Due to summation over only non-negative values, Regression-error@$k(\hat{\v y}_i,\v y_i)\ge0$ always too, which together establish the inequality $0 \le \frac{M}{2}$.

Now, let's prove that $\frac{M}{2} \le \text{XMAD@}2k(\hat{\v y},\v y)$. First, we begin by showing that $\text{Ranking-error@}k(\hat{\v y}_i,\v y_i) \le 2\text{XMAD@}2k(\hat{\v y},\v y)$. Without loss of generality, let's assume that the sets $S(\v y_i,k)$ and $S(\hat{\v y}_i,k)$ are non-overlapping. In the contrary case, the same arguments can be applied to another predicted set $S'$ created by replacing the overlapping labels in $S(\hat{\v y}_i,k)$ with different labels having smaller $\hat{y}_{il}$ values. Thus bounding ranking error on $S'$ will also bound it on $S(\hat{\v y}_i,k)$. Now,

\begin{align}
    &\frac{1}{k}\sum_{l \in S(\v y_i,k)} y_{il} - \frac{1}{k}\sum_{l \in S(\hat{\v y}_i,k)} y_{il}\\
    & \le \frac{1}{k}\sum_{l \in S(\v y_i,k)} y_{il} - \frac{1}{k}\sum_{l \in S(\hat{\v y}_i,k)} y_{il} + \frac{1}{k}\sum_{l \in S(\hat{\v y}_i,k)} \hat{y}_{il} - \frac{1}{k}\sum_{l \in S(\v y_i,k)} \hat{y}_{il}\\
    & \le \frac{1}{k}\sum_{l \in S(\v y_i,k)} e_{il}+\frac{1}{k}\sum_{l \in S(\hat{\v y}_i,k)} e_{il}\mypos{3}\text{where,}\mypos{3}e_{il} = |y_{il}-\hat{y}_{il}|\\
    & \le \frac{1}{k}\sum_{l \in S(\v e_i,2k)} e_{il}\\
    & = 2\text{MAD@}2k(\hat{\v y},\v y)
\end{align}
Bounding the regression error is quite straightforward, hence we skip the proof here.

Finally, the $\text{XMAD@}2k(\hat{\v y}_i,\v y_i) \le \text{XRMSE@}2k(\hat{\v y}_i,\v y_i)$ property follows by using Jensen's inequality on the square function which is concave.
\end{proof}

\begin{theorem}
Given that $y_l = \prod_{h=1}^{H} z_{lh}$ and $\hat{y}_l = \prod_{h=1}^{H} \hat{z}_{lh}$ and under the standard unvisited node assumption of Parabel
\begin{align}
\label{eqn:kld1}
    D_{KL} (y_l || \hat{y}_l)
    \le \sum_{h=1}^H s_{lh} D_{KL}( z_{lh} || \hat{z}_{lh} )~~~\text{where}~~~ s_{lh}=\prod_{\tilde{h}=1}^h z_{l(\tilde{h}-1)}
\end{align}
\end{theorem}
\begin{proof}
We assume that $0\log{\frac{0}{p}} = 0$ where $0 \le p \le 1$.
We also use the unvisited node assumption in Parabel, $\mathbb{P}(z_{lh}=0|z_{l(h-1)}=0)=1$, which means that a child of an unvisited node is never visited.

Let $I_{y_l} \in \{0,1\}$ be the probabilistic variable which says whether label $l$ is relevant to a data point in reference, {\it i.e.} $\mathbb{P}(I_{y_l}=1) = y_l$ and $\mathbb{P}(I_{y_l}=0) = 1-y_l$. Similarly let $I_{z_{lh}} \in \{0,1\}$ be the probabilistic variable which says whether the data point visits node $n_{lh}$ or not, {\it i.e.} $\mathbb{P}(I_{z_{lh}}=1) = z_{lh}$ and $\mathbb{P}(I_{z_{lh}}=0) = 1-z_{lh}$.

Now, since the relevance of label $l$ to a data point is equivalent whether the label path is traversed in the tree by the data point: $y_l=z_{lH}$ and $\mathbb{P}(I_{y_l}) = \mathbb{P}(I_{z_{lH}},\cdots,I_{z_{l1}})$ hold true. 

Due the fact that $x\log{\frac{x}{y}}$ is a convex function, it is easy to show that the KL-divergence between 2 marginal distributions is upper bounded by the KL-divergence of the corresponding joint distributions.

\begin{align}
\label{eqn:kld}
    &D_{KL} (\mathbb{P}(I_{y_l}) || \mathbb{P}(I_{\hat{y}_l})) = D_{KL} (\mathbb{P}(I_{z_{lH}}) || \mathbb{P}(I_{\hat{z}_{lH}}))\\
    &\le D_{KL} (\mathbb{P}(I_{z_{lH}},\cdots,I_{z_{l1}}) || \mathbb{P}(I_{\hat{z}_{lH}},\cdots,I_{\hat{z}_{l1}}))
\end{align}
\begin{align}
    &\text{By using chain rule of KL-Divergence:}\\
    &=D_{KL} (\mathbb{P}(I_{z_{l(H-1)}},\cdots,I_{z_{l1}}) || \mathbb{P}(I_{\hat{z}_{l(H-1)}},\cdots,I_{\hat{z}_{l1}}))\\
    &+ \mathbb{P}(I_{z_{l(H-1)}}=1)D_{KL} (\mathbb{P}(I_{z_{lH}}|I_{z_{l(H-1)}}=1) || \mathbb{P}(I_{\hat{z}_{lH}}|I_{\hat{z}_{l(H-1)}}=1))\\
    &+\mathbb{P}(I_{z_{l(H-1)}}=0)D_{KL} (\mathbb{P}(I_{z_{lH}}|I_{z_{l(H-1)}}=0) || \mathbb{P}(I_{\hat{z}_{lH}}|I_{\hat{z}_{l(H-1)}}=0))
\end{align}
\begin{align}
    &\text{By unvisited node assumption,}\\ &I_{z_{l(H-1)}}=0 \implies I_{z_{lH}}=0 \mypos{3}\text{and}\mypos{3} I_{\hat{z}_{l(H-1)}}=0 \implies I_{\hat{z}_{lH}}=0 \\
    &\text{hence:} \\
    &=D_{KL} (\mathbb{P}(I_{z_{l(H-1)}},\cdots,I_{z_{l1}}) || \mathbb{P}(I_{\hat{z}_{l(H-1)}},\cdots,I_{\hat{z}_{l1}}))\\
    &+ \mathbb{P}(I_{z_{l(H-1)}}=1)D_{KL} (\mathbb{P}(I_{z_{lH}}|I_{z_{l(H-1)}}=1) || \mathbb{P}(I_{\hat{z}_{lH}}|I_{\hat{z}_{l(H-1)}}=1))\\
    &=D_{KL} (\mathbb{P}(I_{z_{l(H-1)}},\cdots,I_{z_{l1}}) || \mathbb{P}(I_{\hat{z}_{l(H-1)}},\cdots,I_{\hat{z}_{l1}})) \\
    &+ (\prod_{\tilde{h}=1}^{H-1} z_{l\tilde{h}}) \Big(z_{lh}\log{\frac{z_{lh}}{\hat{z}_{lh}}} + (1-z_{lh})\log{\frac{1-z_{lh}}{1-\hat{z}_{lh}}}\Big)
\end{align}
\begin{align}
    &\text{By recursively applying above simplification}\\
    & \text{at higher level tree nodes:}\\
    &=\sum_{h=1}^H s_{lh} \Big(z_{lh}\log{\frac{z_{lh}}{\hat{z}_{lh}}} + (1-z_{lh})\log{\frac{1-z_{lh}}{1-\hat{z}_{lh}}}\Big)
\end{align}
The above upper bound is exactly the quantity that \alg minimizes during training by assuming logistic model for probability estimates.
\end{proof}

\begin{lemma}
\alg's overall training objective minimizes an upper bound over XMAD@$k$ for all $k$, with the bound being tighter for smaller $k$ values.
\end{lemma}
\begin{proof}
As presented in the next theorem, \alg minimizes an upper bound on XMAD@$1 = max_{l=1}^L |y_l - \hat{y}_l|$. Note that XMAD@$k \le$ XMAD@$1 \forall k$. Furthermore, as $k$ increases, XMAD@$k$ averages smaller and smaller errors compared the largest errors, therefore the bound is tighter for smaller values of $k$ which are close to $k=1$.
\end{proof}

\begin{theorem}
When each data point has at most $O(\log L)$ positive labels, the expected WP@$k$ regret and XMAD@$k$ error suffered by \alg's pointwise inference algorithm are bounded by:\\
$$O(\log^2 L\sqrt{\frac{W}{\sqrt{Np}}}\sqrt{1+\sqrt{5\log{\frac{3L}{\delta}}} })$$ \\with probability at least $1-\delta$, where $N$ is the total training points, $L$ is the number of labels, $W$ is the maximum norm across all node classifier vectors and $p$ is the minimum probability density of $\v x$ distribution that any tree node receives. 
\end{theorem}
\begin{proof}
The outline of the proof is as follows. First, we see that the WP@$k$ regret and XMAD@$k$ error for a given data point are bounded, in a straight forward manner, by \alg's node and label classifier objectives over that data point. For good overall generalization performance, each classifier needs to receive enough training samples as well as learn to generalize well from those samples. We derive probability bounds for those events simultaneously. While these steps together give the regret bounds for the classifier during exact prediction ({\it i.e.}, calculate the scores for all labels for a given test point), the regret suffered by the greedy, beam-search algorithm might actually be more than that. Therefore, in a follow-up step, we also give a bound for this approximate algorithm which is only worse by $O(\log L)$. This gives us the final sample complexities.

By Lemma~(\ref{lem:main}), both $\frac{1}{2}$WP@$k$ and XMAD@$2k$ are bounded by XRMSE@$2k$ which is in turn bounded by $\max_{l=1}^L |y_l - \hat{y}_l|$.

Now, using Pinsker's inequality~\citep{Krzysztof16},
\begin{align}
    &\max_{l=1}^L |y_l - \hat{y}_l|\\
    &\le\max_{l=1}^L \sqrt{\frac{1}{2}D_{KL}(y_l,\hat{y}_l)}\\
    &=\sqrt{\max_{l=1}^L\frac{1}{2}D_{KL}(y_l,\hat{y}_l)}\\
    &\text{From (\ref{eqn:kld1}):}\\
    &\le\sqrt{\max_{l=1}^L\frac{1}{2}\sum_{h=1}^H s_{lh} D_{KL}( z_{lh} || \hat{z}_{lh} )}\\
    &\le\sqrt{\frac{1}{2}\sum_{n:z_{n-1}>0} s_{n} D_{KL}( z_{n} || \hat{z}_{n} )}\\
    &\text{where $z_{n-1}$ is value in parent of node $n$}\\
\end{align}
\end{proof}

For good generalization performance, we need a small expected regret with respect to distribution over data point $\v x$:

\begin{align}
\label{eqn:r1}
    &\mathbb{E}_x \max_{l=1}^L |y_l - \hat{y}_l|\\
    &\text{By concavity of square root function:}\\
    &\le \sqrt{\frac{1}{2} \mathbb{E}_x \sum_{n:z_{n-1}>0} s_{n} D_{KL}( z_{n} || \hat{z}_{n} )}
\end{align}

Now we try to bound the above quantity by relating it to training error.

Let $p_n$ be the expected fraction of the probability density over $\v x$ that a tree node $n$ receives. This is precisely the density of data points which have at least one label with non-zero relevance in the subtree rooted at node $n$. Now, let's compute the probability that the node $n$ receives at least $Np_n(1-k)$ training points where $N$ is the number of total training points and $Np_n$ is the expected number of training points that node $n$ would receive. By using chernoff bound, this probability is at least $1 - \exp(-\frac{p_nNk^2}{2})$. Now, the probability that all tree nodes $n$ would simultaneously receive at least $Np_n(1-k)$ training points is at least $1-L\exp(-\frac{pNk^2}{2})$ since there are at most $L$ tree nodes and each has $\v x$ density of at least $p$.

Now, we use the result in~\citep{Kakade09}. Since the logistic loss used for modeling probabilities in \alg is lipschitz continuous with constant $1$ and logistic regression parameters are bounded by norm $W$, and $\v x$ is bounded by norm $1$, for any regressor in \alg,
\begin{align}
    \mathbb{E}_x s_n D_{KL}(z_n,\hat{z}_n) \le \hat{\mathbb{E}}_x s_n D_{KL}(z_n,\hat{z}_n) + 2W\sqrt{\frac{1}{Np(1-k)}} + 2W\Delta
\end{align}
with probability at least $1-\exp(-2Np(1-k)\Delta^2)$ where $\hat{\mathbb{E}}_x D_{KL}(z_n,\hat{z}_n)$ is the average training error in node $n$ which is $0$ as per our assumption.

Combining the above reasonings, along with the fact that there are at most $2L$ regressors in \alg, we can conclude that with probability of at least $1-L\exp(-\frac{pNk^2}{2})-2L\exp(-2Np(1-k)\Delta^2)$, each node has expected error bounded simultaneously as below:
\begin{align}
    \mathbb{E}_x s_n D_{KL}(z_n,\hat{z}_n) \le 2W\sqrt{\frac{1}{Np(1-k)}} + 2W\Delta
\end{align}

Now, note that $k$ can be given any value in $[0,1]$ and the above bounds vary accordingly. We choose to give $k=2\Delta(\sqrt{\Delta^2+1}-\Delta)$. Then, with probability at least $1-3L\exp(-2(\sqrt(2)-1)^2Np\Delta^2)$, for all regressors
\begin{align}
    \mathbb{E}_x s_n D_{KL}(z_n,\hat{z}_n) \le 2W\sqrt{\frac{1}{Np(1-2\Delta(\sqrt(\Delta^2+1)-\Delta))}} + 2W\Delta
\end{align}
In other words, with probability at least $1-\delta$ over the training samples, for all regressors,
\begin{align}
    &\mathbb{E}_x s_n D_{KL}(z_n,\hat{z}_n) \le 2W\sqrt{\frac{1}{Np(1-2\Delta(\sqrt(\Delta^2+1)-\Delta))}}\\
    &+ 2W\sqrt{\frac{1}{2(\sqrt{2}-1)^2Np}\log\Big(\frac{3L}{\delta}\Big)}
\end{align}
where $\Delta=\sqrt{\frac{1}{2(\sqrt{2}-1)^2Np}\log\Big(\frac{3L}{\delta}\Big)}$. Now since $\Delta \to 0$ as $N\to\infty$, for large enough $N$, the above bound can be approximated to

\begin{align}
    \mathbb{E}_x s_n D_{KL}(z_n,\hat{z}_n) \le 2W\sqrt{\frac{1}{Np}} + 2W\sqrt{\frac{1}{2(\sqrt{2}-1)^2Np}\log\Big(\frac{3L}{\delta}\Big)}
\end{align}

From~(\ref{eqn:r1}),
\begin{align}
\label{eqn:r1}
    &\mathbb{E}_x \max_{l=1}^L |y_l - \hat{y}_l|\\
    &\le \sqrt{\frac{1}{2} \mathbb{E}_x \sum_{n:z_{n-1}>0} s_{n} D_{KL}( z_{n} || \hat{z}_{n} )}\\
    &\text{Since any $\v x$ has on average $\log L$ non-zero labels and}\\
    &\text{since height of the tree is $\log L$}\\
    &\text{the number of nodes with $z_{n-1}>0$ for any $\v x$ is on average $\log^2 L$, hence:}\\
    &\le \sqrt{\frac{\log^2 L}{2} \Big(2W\sqrt{\frac{1}{Np}} + 2W\sqrt{\frac{1}{2(\sqrt{2}-1)^2Np}\log\Big(\frac{3L}{\delta}\Big)} \Big)}\\
    &\le\log L \sqrt{\frac{W}{\sqrt{Np}}}\sqrt{1+\sqrt{5\log\Big(\frac{3L}{\delta}\Big)}}
\end{align}
with probability at least $1-\delta$ over training samples.

The above bound holds for exact prediction where all label probabilities are computed for a given test point. Now we analyse the extra regret due to the greedy, approximate, beam search based, pointwise inference algorithm used by \alg.

During beam-search, a point traverses the tree level-by-level. At each tree level, a small shortlist of around $k=10$ most probable nodes, {\it i.e.} nodes with most relevant labels their subtrees, are maintained and extended on to next level. If accurate label relevances were available, then beam search would always return the best set of labels, since each node's $z_n$ variable value matches the most relevant label in its subtree. Unfortunately, due to generalization error, the estimated $\hat{z}_n$ values might not exactly match the $Z_n$ values. As a result, the regret accumulates at each tree level whenever a node with lower $z_n$ is maintained in shortlist instead of the highest one. The regret suffered is at most $\max_{n \in S} 2|z_n - \hat{z}_n|$, where $S$ is the set of shortlisted nodes at a tree level. A little more algebra reveals that this quantity is in fact bounded by (\ref{eqn:r1}).

\begin{align}
    &\max_{n \in S} 2|z_n - \hat{z}_n| \le 
    \sqrt{\frac{1}{2} \mathbb{E}_x \sum_{n:z_{n-1}>0} s_{n} D_{KL}( z_{n} || \hat{z}_{n} )}
\end{align}

which is the bound on the regret suffered by exact prediction algorithm. That is, beam-search can suffer at most the same amount of regret at each tree level that exact prediction suffers as a whole. Now since there are $\log L$ tree levels, the regret of beam search algorithm is bounded by
\begin{align}
    \le\log^2 L \sqrt{\frac{W}{\sqrt{Np}}}\sqrt{1+\sqrt{5\log\Big(\frac{3L}{\delta}\Big)}}
\end{align}

\bibliographystyle{ACM-Reference-Format}


\end{document}